\documentclass{article}

\usepackage[compress]{natbib}
\usepackage[letterpaper, top=1.0in, bottom=1.0in, left=1.5in, right=1.5in]{geometry}
\usepackage[utf8]{inputenc} %
\usepackage[T1]{fontenc}    %
\usepackage{lmodern}
\usepackage{hyperref}       %
\usepackage{url}            %
\usepackage{booktabs}       %
\usepackage{amsfonts}       %
\usepackage{nicefrac}       %
\usepackage{microtype}      %

\usepackage[symbol]{footmisc}  %
\renewcommand{\thefootnote}{\fnsymbol{footnote}}
\newcommand\blfootnote[1]{%
  \begingroup
  \renewcommand\thefootnote{}\footnote{#1}%
  \endgroup
}
\renewcommand{\thefootnote}{\arabic{footnote}}

\makeatletter
\newcommand{\sparagraph}{\@startsection{paragraph}{4}{\z@}%
                                    {2.25ex \@plus0.5ex \@minus0.2ex}%
                                    {-1em}%
                                    {\normalfont\normalsize\bfseries}}

\makeatother

\input{math_macros}

\usepackage{algorithmic} %
\usepackage{algorithm} %
\usepackage[capitalize]{cleveref}
\usepackage{enumitem}  %
\usepackage{tabularx,array,makecell} %
\usepackage{subcaption} %
\usepackage[flushleft]{threeparttable}  %
\newcommand{\hp}{\hphantom{0}} %

\usepackage{mathtools}  %

\newcommand{\uniformCorruptions}{additive query-independent corruptions}
\newcommand{\UniformCorruptions}{Additive query-independent corruptions}
\newcommand{\approvalrl}{approval}
\newcommand{\Approvalrl}{Approval}
\newcommand{\videolink}{https://youtu.be/aJsJzrrfZBM}

\usepackage{todonotes}

\font\titlefont=cmr12 at 15.5pt
\title{{\titlefont Avoiding Tampering Incentives in Deep RL \\
via Decoupled Approval}}

\author{
  Jonathan Uesato$^*$ \vspace{-0.05in} \\ \and
  Ramana Kumar$^*$ \vspace{-0.05in} \\ \and
  Victoria Krakovna \vspace{-0.05in} \\ \and
  Tom Everitt \hp \and
  Richard Ngo \hp \and
  Shane Legg \hp\hp\hp %
}
\date{}

\begin{document}

\maketitle

\begin{abstract}
How can we design agents that pursue a given objective when all feedback mechanisms are influenceable by the agent?
Standard RL algorithms assume a secure reward function, and can thus perform poorly in settings where agents can tamper with the reward-generating mechanism.
We present a principled solution to the problem of learning from influenceable feedback, which combines approval with a decoupled feedback collection procedure.
For a natural class of corruption functions, decoupled approval algorithms have aligned incentives both at convergence and for their local updates.
Empirically, they also scale to complex 3D environments where tampering is possible.
\end{abstract}

\section{Introduction}
\label{sec:intro}

\blfootnote{$^*$ Equal contribution.}
\blfootnote{DeepMind, London, UK. Correspondence to \texttt{\{juesato, ramanakumar\}@google.com}}
\setcounter{footnote}{0}
If reinforcement learning (RL) agents are to have a large influence in society, it is essential that we have reliable mechanisms to communicate our preferences to these systems.
In the standard RL paradigm, the role of communicating our preferences is played by the reward function.
However, it may not be possible to restrict sufficiently general RL agents from modifying physical implementations of their reward function, or more generally tampering with whatever process produces inputs to the learning algorithm, instead of pursuing the intended goal.
Our central concern is the \emph{tampering problem}, which can be summarized as:
\begin{center}
\emph{How can we design agents that pursue a given objective when all feedback mechanisms for describing that objective are influenceable
by the agent?}
\end{center}

As a simplified example, consider designing an automated personal assistant with the objective of being useful for its user.
A natural approach is to regularly query the user on their satisfaction with the system, and to optimize the expected user satisfaction.
This system may learn behaviors that alter the user's preferences to be easier to satisfy over time -- for example, towards a disposition for assigning high ratings, or preferences described by simple rules.
How could we design a system that instead conforms to the user's current preferences, without being incentivized to influence them?

Whether an RL system engages in tampering depends on both its capability to tamper and its incentives.
Even if current ML systems do not tamper, the same incentives may lead to tampering in more capable systems.
This motivates us to develop theory that is able to provide principled reasons why ML systems will avoid tampering in the future even as they become capable of increasingly complex reasoning.
We do this by formalizing what we mean by incentives, so we can establish formal guarantees on the behaviors encouraged by the learning algorithms we consider.
At the same time, these algorithms should be scalable -- able to solve the same tasks as current ML algorithms at similar cost.

We consider a simplified setting where agents learn from \emph{approval} feedback, i.e., information about the supervisor's preferred action, so that feedback need only be optimized myopically (\S\ref{sec:formalism}).
Even in this simplified setting, standard RL algorithms have an incentive to tamper with the feedback (\S\ref{sec:ad}).
To solve this problem, we consider \emph{decoupled approval} algorithms, which build on \emph{approval-direction} \citep{Christiano2014approval} and, more broadly, \emph{decoupled RL} \citep{everitt2018thesis}.
Intuitively, we would like the action being optimized to always be independent of corruptions to the feedback used to evaluate that action.
Decoupled approval proposes a specific mechanism to achieve this, by only optimizing for the feedback on a \emph{query} action, sampled independently from the action taken in the world.

We conduct extensive experiments in REALab \citep{kumar2020realab}, a 3D environment designed for studying tampering.
Our experimental results strongly support our theoretical claims across a broad range of agent designs.
In particular, while standard RL algorithms result in tampering, our decoupled approval agents avoid tampering and maintain strong task performance.
To our knowledge, our work is the first to empirically study tampering incentives when using neural network policies in the deep RL setting.

\paragraph{Our contributions}
We propose two algorithms, DA-PG and DA-QL, that apply decoupled approval to policy gradients and Q-Learning.
Theoretically, under the assumption that tampering actions affect all queries equally, we show that standard RL algorithms incentivize tampering, while our decoupled approval algorithms do not. %
We provide empirical evidence supporting our theoretical claims, and demonstrate favorable task performance and tampering behavior for decoupled approval agents in the deep RL setting.
As a result, these algorithms are the first to have both \emph{theoretical guarantees} against tampering and \emph{empirical evidence} of their compatibility with standard deep RL techniques.

\section{Technical Preliminaries}
\label{sec:formalism}

\paragraph{Corrupt feedback MDPs}
The general problem we study is how a user can reliably communicate a task to an agent, even though any piece of information can in principle be tampered with by the agent.
To study this problem formally, we use corrupt feedback MDPs (CFMDPs, \citep{kumar2020realab}).
CFMDPs extend Markov Decision Problems with feedback corruption (to allow tampering), and general feedback (to enable a broader class of learning algorithms).
For ease of reference, we briefly summarize the CFMDP formalism and introduce some notation.

First, an MDP $(\mcS, \mcA, p, f, r, \gamma)$ models the underlying environment and the user's intended task.
Here
$\mcS$ and $\mcA$ are sets of states and actions,
$p$ an initial state distribution,
$f$ a stochastic transition function, and
$r:\mcS\times\mcA\to\mathbb{R}$ a reward function describing the intended task.

On each step, the user provides the agent \emph{feedback}, a generalization of rewards.
Formally, the agent can submit a query $K_t\in\mcK$ to the user at each
time step, to which the user replies with feedback $D_{t+1} = \delta(S_t, K_t)\in \mcD$.\footnote{Choosing $\mcK = \mcA, \mcD = \mathbb{R}$, and $\delta = r$ corresponds to providing rewards as feedback.}
However, the agent only observes the corrupted feedback $\tilde D_{t+1} = c(S_{t+1}, K_t, D_{t+1})$, where $c: \mcS\times\mcK\times\mcD \to \mcD$ is the corruption function.

\begin{definition}[Corrupt feedback MDP]
  A CFMDP is a tuple $\mu=(\mcS, \mcA, \mcD,\mcK, p, f, r, \gamma, c)$.
\end{definition}

Corruption is measured by the difference between the true feedback $D_t$ and the agent's observed feedback $\tilde{D_t}$.
While agents are ultimately evaluated according to total reward, effective learning will often require designing agents which avoid corruption.

\paragraph{Approval feedback}
CFMDPs allow agents to learn from different forms of feedback, such as instantaneous rewards \citep{sutton1998introduction} or expert demonstrations \citep{ng2000algorithms, ross2011reduction}.
In this work, we focus on learning from \emph{approval} of agent actions.

\begin{definition}[Approval feedback, approval-optimal policy]
For \emph{approval feedback}, queries are expressed as actions $\mcK = \mcA$, and the feedback function $\delta$ expresses human approval for the query action.
If in state $s$, the human prefers the agent take action $a_1$ over $a_2$, then $\delta(s, a_1) > \delta(s, a_2)$.
When using approval feedback, an \emph{approval-optimal} policy is any policy $\pi^* \in \arg\max_{\pi \in\Pi} \ev_\pi\br{\delta(s, a)}, \forall s \in S$, where $\Pi$ is the set of all stationary policies from states to distributions over actions.
\end{definition}

Similar to \emph{value advice} feedback \citep{knox2009tamer,Daswani2014ValueAdvice}, but in contrast to instantaneous reward feedback, approval feedback for an action accounts for its long-term consequences.
Compared to value advice feedback, which provides a numerical estimate of total discounted rewards, approval feedback is only assumed to indicate \emph{relative} preferences between different actions.\footnote{Because value advice feedback will be consistent with relative preferences between actions, all our algorithms and results also apply to the value advice setting.}
Our theoretical results will assume that the user knows the environment well enough to give accurate approval, such that myopically optimizing expected approval implies optimizing expected discounted true returns.
We discuss the practical considerations of this assumption in \S\ref{subsec:practical_consid}.
Thus, our theoretical results focus on learning an approval-optimal policy.
In our experiments, we use a synthetic approval signal that satisfies this assumption, so approval-optimal policies also optimize the total return.

\paragraph{\UniformCorruptions}
With arbitrary corruption functions, learning is impossible \citep[Thm.~11]{Everitt2017CRMDP}.
For example, if the corruption function always overwrites the feedback with a constant, the learner receives no approval information.
To avoid this issue, we focus on \uniformCorruptions:

\begin{definition}[\UniformCorruptions]
\label{def:uniform_corruption}
  In CFMDPs with real-valued feedback, an \emph{additive, query-independent} corruption function can be expressed as $c(s, k, d) = d + c_s$, where there is a fixed offset $c_s\in \R$ for each state $s$.
\end{definition}

\UniformCorruptions{} encompass one of the simplest ways an agent can tamper with its feedback: taking an action that uniformly translates the numerical value of the feedback.
For example, consider an agent able to press a button that increments its feedback score.
The amount of corruption depends indirectly on the action the agent took (via the dependence on the state), but cannot depend on the agent's query.
Many realistic forms of corruption are not included under this assumption (see \S\ref{sec:discussion}). %
However, we focus on \uniformCorruptions{} both to simplify our analysis and because even under this restricted assumption, tampering incentives prevent standard algorithms from learning an optimal policy.
By contrast, decoupled approval agents have no incentive to tamper, despite that option being available.

\subsection{Two Types of Incentives}
\label{sec:incentive_defs}

Intuitively, our aim is to avoid algorithms that optimize the wrong objective.
But what does it mean for an algorithm to optimize an objective at all, and how do we determine which objective is being optimized?
In this paper, we consider two types of answers, which are instances of the standard notions of convergence to an optimal policy and policy improvement (both with respect to approval).
\emph{Convergence incentives} relate to behavior in the limit: setting aside questions of exploration and generalization, how will the policy behave?
\emph{Update incentives} relate to the local behavior of the algorithm before convergence: at each step, which direction does the algorithm push the policy?
We define both properties for sequences of policies $\pi_t:\mcS\leadsto\mcA$ produced by an iterative algorithm.%

\begin{definition}[Aligned convergence incentives]
A learning algorithm incentivizes policy $\pi$ iff $\lim_{t\to\infty}\pi_t = \pi$ (almost surely).
If $\pi$ is also approval-optimal, the algorithm is said to induce \emph{aligned convergence incentives}.
\end{definition}

\begin{definition}[Aligned update incentives]\label{def:local-incentives}
For a policy parameterized by $\theta_t$, a learning algorithm induces \emph{aligned update incentives} iff
  $\forall s, t: \overline{R}(\overline{\theta_{t+1}}) \geq \overline{R}(\theta_t)$,
where $\overline{R}(\theta) = \ev_{k\sim\pi_\theta(\cdot\mid s)}\br{\delta(s, k)}$ is the expected approval,
  and $\overline{\theta_{t+1}} = \ev\br{\theta_{t+1} \mid \theta_t, s}$ is the expectation of the updated parameters.
\end{definition}

In other words, aligned update incentives mean that for all states, the expected update increases expected approval for that state.
Intuitively, this means that the expected update shifts probability away from tampering actions (that receive higher observed feedback but lower true feedback).

Informally, we will use the language of \emph{incentives} to refer to the types of behavior that tend to result from policies produced by a learning algorithm.
While our formalizations match the standard notions of convergence and policy improvement in the non-corrupt feedback setting, our emphasis is different: incentives focus on the \emph{target} of optimization %
rather than its efficacy.
In the non-corrupt feedback setting, the optimization target can often be assumed to match the ideal target.
However, naive applications of algorithms which optimize effectively when given non-corrupt feedback can result in policies which perform worse than randomly for the corrupt feedback setting due to optimizing the wrong objective, as we shall see.

\section{Decoupled Approval}
\label{sec:ad}
\label{sec:algorithms}

At a high level, decoupled approval involves two ideas: learning from approval with myopic optimization, and decoupling.
We describe these here before presenting concrete algorithms. %

\paragraph{Learning from human approval with myopic optimization}
In methods based on approval such as TAMER \citep{knox2008tamer} and COACH \citep{celemin2015coach}, the feedback signal is human approval. %
At each step, the agent queries the user about a potential action, and the user provides real-valued feedback for the overall desirability of this action.
Since the feedback signal already captures overall consequences, as opposed to instantaneous rewards, approval should be optimized myopically, i.e., at each step, take the action with highest expected approval.
This means off-the-shelf RL algorithms can be adapted as approval-based methods by treating approval as reward and setting discount $\gamma=0$.

\paragraph{Decoupling}
In standard (non-decoupled) algorithms, the agent only gets feedback for the state-action pair $(s, a)$ when it takes action $a$ in state $s$.
Thus if taking action $a$ in state $s$ tampers with the feedback, the agent never gets uncorrupted feedback about this.
We use the term \emph{decoupled RL} to refer to algorithms that address this issue by learning from feedback on state-action pairs different from the current one \citep{Everitt2017CRMDP}.
The amount of decoupling can vary: whereas \citet{Everitt2017CRMDP} gives feedback on a random state, fully decoupled from the current state, approval-directed agents \citep{Christiano2014approval} use feedback on the current state about a query action, which may be distinct from the action taken in the environment.
Intuitively, this can increase learning efficiency by focusing feedback on states from the agent's current state visitation distribution.
Our approach, decoupled approval, is an instance of approval direction with a specific choice for the query policy: we use the same policy (up to exploration rates) for both query and taken actions, but sample the two actions independently. \\

\noindent \emph{Intuition for decoupled approval.}
To provide intuition for decoupled approval, we compare it to an (impractical) algorithm we call Human-in-the-loop Approval, shown in Figure \ref{fig:da_diagrams}.
In this algorithm, at each step, the agent sends a text description proposing an action to the human supervisor.
The human supervisor assigns an approval score to this action, only after which the action is taken by the agent.
The agent optimizes these approval scores.
Optionally, we could allow the agent to take no action if the approval is too low.
But even if the agent takes a tampering action it will have first received uncorrupted feedback on that action, since the feedback is received \emph{before} the actions are executed, and thus before tampering can occur.
Decoupled approval achieves a similar effect while allowing the supervisor to evaluate queries asynchronously.
Because the query and taken actions are chosen independently, queries about tampering actions will also be independent of when tampering actually occurs, and so these queries will not be reinforced. \\

\noindent \emph{Importance of decoupling.}
To isolate the effect of decoupling, consider an agent myopically optimizing for approval in a CFMDP.
In this setting, both decoupled approval and standard RL algorithms select actions to maximize approval.
The only difference is decoupling: standard RL agents always receive feedback about the action they take, whereas decoupled approval agents sample the query and taken actions independently (and optimize the policy using feedback on the queries).

\begin{remark}[Standard RL optimizes a corrupted signal]\label{remark:myopic}
Standard RL algorithms effectively treat each CFMDP as an MDP with reward function given by corrupted feedback.
At each step $t$, the agent observes state $S_t$, takes action $A_t$, and queries with $K_t = A_t$.
The agent then treats the observed approval $c(S_{t+1}, A_t, \delta(S_t, A_t))$ as the reward.
Assuming \uniformCorruptions, this simplifies to $c_{S_{t+1}} + \delta(S_t, A_t)$.
Optimizing this reward myopically with $\gamma=0$, the optimal policy is $\pi(s) = \argmax_a \ev\br{c_{S_{t+1}} \mid S_t = s, A_t = a} + \delta(s, a)$.
\end{remark}

From the perspective of the MDP learning algorithm, the policy successfully optimizes the given reward function.
However, from the perspective of the agent designer, the policy does not optimize the intended approval signal -- if the scale of $c_s$ is much larger than that of $\delta$, this policy could perform poorly with respect to $\delta$.
We discuss this in more detail in Appendix \ref{app:myopic-tampering}.
By adding decoupling, we will see that decoupled approval algorithms converge to $\pi(s) = \argmax_a \delta(s, a)$ as desired.

\begin{figure}[t]
\centering
\begin{subfigure}[b]{0.48\textwidth}
\includegraphics[width=\textwidth]{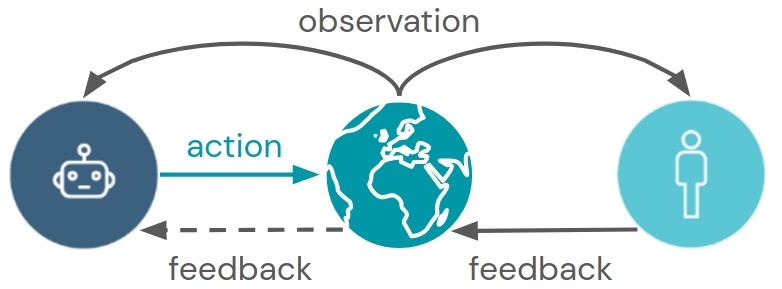}
\caption{\Approvalrl{} RL}
\label{subfig:standard_diagram}
\end{subfigure}\hfill
\begin{subfigure}[b]{0.48\textwidth}
\includegraphics[width=\textwidth]{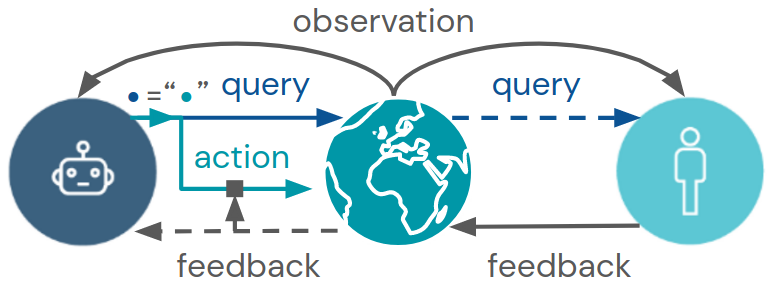}
\caption{Human-in-the-loop Approval}
\label{subfig:hitl_diagram}
\end{subfigure} \\[1ex]
\begin{subfigure}[b]{0.48\textwidth}
\includegraphics[width=\textwidth]{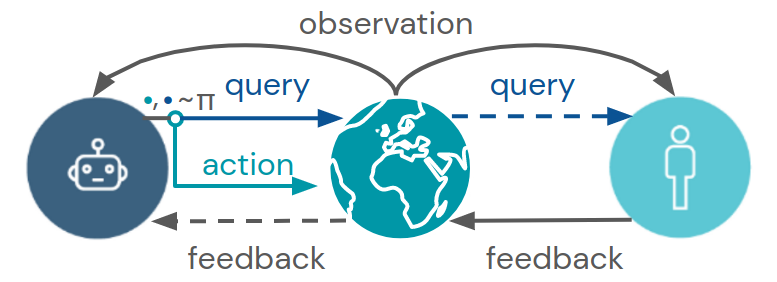}
\caption{Decoupled Approval}
\label{subfig:da_diagram}
\end{subfigure}
\caption{\textbf{Intuition for Decoupled Approval.}
In Standard / Approval RL (\ref{subfig:standard_diagram}), the agent has an incentive to take actions which lead to corrupting states producing high feedback.
Human-in-the-loop approval (\ref{subfig:hitl_diagram}) avoids this incentive by sending query actions to a human supervisor, who provides feedback on these actions \emph{before} they are executed, thus preventing the action from corrupting its own feedback.
Decoupled approval (\ref{subfig:da_diagram}) achieves similar benefits, but without a human-in-the-loop, by optimizing for the feedback on the query action, sampled independently from the action taken in the world.
}
\label{fig:da_diagrams}
\end{figure}

\subsection{Decoupled Approval Policy Gradients}

\label{subsec:adpg}

Decoupled approval policy gradients (DA-PG) extend policy gradients by independently sampling an action $a$ and a query $k$ from the same policy.
We provide pseudocode for DA-PG in Algorithm \ref{alg:adpg}.
Note that Algorithm \ref{alg:adpg} reduces exactly to the standard policy gradient algorithm (with approval feedback) if we take $k = a$ in Line \ref{line:adpg_diff}, rather than sampling both independently.

\begin{algorithm}
\caption{Decoupled Approval Policy Gradients (DA-PG)}
\label{alg:adpg}
\begin{algorithmic}[1]
\STATE Initialize the policy parameters $\theta_0=0$
  \FOR {$t=0$ {\bfseries to} $T$}
    \STATE Observe current state $s$
    \STATE Take action $a \sim \pi_{\theta_t}(s)$
           and query $k \sim \pi_{\theta_t}(s)$
           \label{line:adpg_diff}
    \STATE Receive next state $s' \sim f(s, a)$ and corrupted approval $\tilde d = c(s', k, \delta(s, k))$
    \STATE Update parameters via policy gradient:
    $\theta_{t+1} := \theta_t + \alpha \tilde d \nabla_\theta \log \pisk$
\ENDFOR
\end{algorithmic}
\end{algorithm}

\begin{proposition}
\label{obs:adpg_unbiased}
  In any state $s$, for any policy $\pi_\theta$, the expected update provided by DA-PG in a CFMDP with \uniformCorruptions{} is equivalent to the expected update provided by approval-based policy gradients in the corresponding uncorrupted MDP $\bar{\mu} = (\mcS, \mcA, p, f, \delta, \gamma)$.
That is,
\begin{equation*}
\ev_{\mu, \pi_\theta}\br{c(s', k, \delta(s, k)) \nabla_\theta \log \pisk}
= \; \ev_{\bar{\mu}, \pi_\theta}\br{\delta(s, a) \nabla_\theta \log \pisa}
\end{equation*}
\end{proposition}
\begin{proof}
Assuming \uniformCorruptions, we have
\begin{align*}
&\; \ev_{\mu, \pi_\theta}\br{c(s', k, \delta(s, k)) \nabla_\theta \log \pisk} \\
=&\; \ev_{\mu, \pi_\theta}\brackets{(\delta(s, k) + c_{s'}) \nabla_\theta \log \pisk} \\
=&\; \ev_{\bar{\mu}, \pi_\theta}\brackets{\delta(s, a) \nabla_\theta \log \pisa} +
       \ev_{\mu, \pi_\theta}\brackets{c_{s'}} \ev_{\mu, \pi_\theta}\brackets{\nabla_\theta \log \pisk} \\
=&\; \ev_{\bar{\mu}, \pi_\theta}\brackets{\delta(s, a) \nabla_\theta \log \pisa}
\end{align*}
where the second equality uses the fact that $a$ and $k$ are sampled independently, and the third equality holds because the expectation of the score function is 0.
\end{proof}

Intuitively, the expected policy gradient depends only on the advantages at each step, and \uniformCorruptions{} preserve the relative advantage between any two actions, regardless of behavior policy.
We relate this property to incentives.

\begin{corollary}
  \label{cor:da-pg}
Consider policies parameterized by a mixture parameter $z$ over two experts $\pi_1$ and $\pi_2$: $\pi_z(a\mid s) = \sigma(z)\pi_1(a\mid s) + (1 - \sigma(z))\pi_2(a\mid s)$, where $\sigma(z) = (1+\exp(-x))^{-1}$ is the sigmoid function.
DA-PG applied to this policy parameterization induces aligned update incentives.
\end{corollary}

Because the DA-PG update is equal to the standard policy gradient update in expectation, Proposition \ref{obs:adpg_unbiased} also implies the gradient converges to zero, under the same conditions as Theorem 3 in \citet{sutton2000policy}.
In practical cases, this optimization will be non-convex so we cannot prove convergence incentives.
However, when the policy parameterization induces concave expected returns, we obtain convergence to a global optimum and thus aligned convergence incentives.

\emph{Robustness to scale.}
Notably, incentives for DA-PG do not rely on $\delta$ encoding any information about the corruption function $c$.
We refer to such evaluators as \emph{tampering-agnostic}.
This is important because it is typically difficult for the evaluator to recognize all possible forms of tampering.
In the language of \citet{Demski2019embedded}, decoupled approval exhibits \emph{robustness to relative scale} and to scaling up.
We contrast decoupled approval to approaches which rely on explicit feedback alone to discourage tampering \citep{saunders2017trial,reddy2019learning}, which may fail if the policy can discover forms of tampering undetected by the evaluator, though we believe these approaches complement decoupled approval, as noted in \S\ref{subsec:other_approaches}.
Our experiments in \S\ref{sec:experiments} will support robustness to scale of decoupled approval by using a tampering-agnostic evaluator.

\subsection{Decoupled Approval Q-Learning}
\label{subsec:adql}

DA-QL extends standard Q-learning by sampling query and taken actions independently, similarly to DA-PG.
The other modification is the introduction of an importance sampling correction (Line \ref{line:adql_is} in Algorithm \ref{alg:adql}),
which counterbalances the effect of Q-values being updated for some actions more frequently depending on the query policy.
Appendix \ref{app:importance-sampling} discusses a surprising failure mode, in which DA-QL converges to tampering policies when this correction is omitted.

\begin{algorithm}
\caption{Decoupled Approval Q-Learning (DA-QL)}
\label{alg:adql}
\begin{algorithmic}[1]
\STATE Set initial Q-values $Q_0(s,k)=0$ for all $s,k$
\STATE Set visit count $M_0(s_0) = 1$ for initial state $s_0$, and $M_0(s)=0$ for all other states $s\not=s_0$
\FOR {$t=0$ {\bfseries to} $T$}
    \STATE Observe current state $s$
    \STATE $\pi_A := \epsilon\textrm{-greedy}(Q_t), \epsilon = 1/M_t(s)$
    \STATE $\pi_K := \epsilon\textrm{-greedy}(Q_t), \epsilon = \max(1/M_t(s),  \lrscale \size{A})$
    \STATE Take action $a \sim \pi_A(s)$
           and query $k \sim \pi_K(s)$
    \STATE $\alpha := \lrscale(M_t(s) \pi_K(k \mid s))^{-1}$ \label{line:adql_is}
    \STATE Receive next state $s' \sim f(s, a)$ and corrupted approval $\tilde d = c(s', k, \delta(s, k))$
    \STATE $Q_{t+1}(s, k) := (1 - \alpha) Q_t(s, k) + \alpha \tilde d$
    \STATE $M_{t + 1}(s') := M_t(s') + 1$
\ENDFOR
\end{algorithmic}
\end{algorithm}

We first analyze the local properties of DA-QL in terms of the expected update at each step.
\begin{proposition}
\label{prop:adql_local_incentives}
In DA-QL under \uniformCorruptions, for all $s, k, t$, we have
$$\ev\br{\xi_t(k) \mid H_t} =   h_1 (\delta(s,k) - Q_t(s,k)) + h_2$$
where $\xi_t(k) = Q_{t+1}(s,k) - Q_t(s,k)$, $H_t$ is the history up to time $t$ (including the current state $s$ but not $a$ or $k$), and $h_1, h_2$ are functions of $H_t$.
In particular, whenever $\delta(s, k_1) > \delta(s, k_2)$, at least one of the following holds:
\begin{equation*}
\ev\br{\xi_t(k_1)\mid H_t} > \ev\br{\xi_t(k_2)\mid H_t} \;\;\;\; or \;\;\;\;
Q_t(s, k_1) > Q_t(s, k_2)
\end{equation*}
\end{proposition}
\begin{proof}
Refer to Appendix \ref{app:adql-local}.
\end{proof}
Proposition \ref{prop:adql_local_incentives} establishes an analogue of aligned local update incentives for DA-QL in the Q-learning setting (where the parameterization of Definition \ref{def:local-incentives} is not a natural fit).
Whenever the user prefers $k_1$ over $k_2$, the expected update increases $Q(s, k_1)$ relative to $Q(s, k_2)$, unless $Q$ already prefers $k_1$. %

We now analyze the global properties of DA-QL.
Our main result states that DA-QL converges to an approval-optimal policy, and thus induces aligned convergence incentives:
\begin{proposition}\label{thm:adql}
Let $k, k' \in \mcK$ be two distinct query actions.
For DA-QL in CFMDPs with \uniformCorruptions, the following expression converges almost surely to 0 as $t\to\infty$:
\begin{align*}
\left(Q_t(s,k)-Q_t(s,k')\right)-
\left(\delta(s,k)-\delta(s,k')\right)
\end{align*}
\end{proposition}
\begin{proof}[Proof Sketch]
Our proof follows the convergence proof for standard Q-Learning from ~\citep{Jaakkola1993convergence}.
The value of the expression above after each DA-QL update decomposes into a contraction term of its previous value, and a term which depends on the observed feedback.
With the assumption of \uniformCorruptions{} and the importance sampling correction, this second term has expectation zero.
Appropriately chosen learning rates also ensure bounded variance.
We can then instantiate the stochastic convergence theorem from \citet{Jaakkola1993convergence} for the above expression.
The full proof is in Appendix \ref{app:adql-global}.
\end{proof}
\begin{remark}[DA-QL recovers optimal policy]
Let the action ranking induced by a function $f(s, k)$ for state $s$ be the list of actions $k$ in decreasing order of $f(s,k)$.
Since the action-value differences $Q_t(s,k)-Q_t(s,k')$ converge to the feedback differences $\delta(s,k)-\delta(s,k')$, the action ranking for each state induced by $Q_t$ converges to the action ranking induced by the feedback function $\delta$.
Thus, a policy that maximizes $Q_t(s,\cdot)$ will converge to an approval-optimal policy.
\end{remark}

\section{Experiments}
\label{sec:experiments}

\subsection{Embedded Feedback Tasks}
\label{subsec:realab}
To study tampering empirically, we use REALab \citep{kumar2020realab}, a platform for simulated environments with \emph{embedded feedback} to model the tampering problem.
For ease of reference, we briefly describe REALab here, but refer readers to \citet{kumar2020realab} for details.
In REALab, all mechanisms for providing feedback to the agent are potentially influencable by the agent (as in CFMDPs) because they are made from influenceable components, primarily \emph{registers}.
In REALab, as in CFMDPs, agent designers specify a feedback function, which writes the feedback (such as instantaneous rewards or approval) to a \emph{feedback register}.

Registers are physical objects used to store a value by representing it as a physical property.
Our experiments use \texttt{TwoBlockRegister}s composed of a base and offset block, which encode values using the distance between the two blocks (see \S3.1 of \citet{kumar2020realab}).
Crucially, values stored in registers can be corrupted by physical interactions, such as the agent pushing or throwing the register blocks.
For \texttt{TwoBlockRegister}s, tampering with the offset block will corrupt the observed feedback for the next timestep, but tampering with the base block will corrupt the observed feedback for all subsequent timesteps.

Other than the use of embedded feedback, REALab tasks are standard and similar to tasks in DMLab-30 \citep{beattie2016deepmind}.
Each task requires controlling an avatar with a first-person view of the environment to collect apples within some fixed time duration.
We primarily evaluate on the \emph{Unlock Door} task.
In this task, there are small apples providing reward 1 within a large room.
Large apples provide reward 10, and can be accessed by standing on a sensor, which unlocks a door to a room containing the apple.
Because collecting the large apple ends the episode, the optimal strategy is to first gather the small apples, before visiting the sensor and collecting the large apple.
We also include additional experiments on a \emph{Seek-Avoid} task and on tabular CFMDPs in Appendices \ref{app:seekavoid} and \ref{app:tabular_cfmdp}.

\subsection{Agent Implementation}
\label{subsec:exp_design}

We evaluate both DA-PG (Algorithm \ref{alg:adpg}) and DA-QL (Algorithm \ref{alg:adql}), when combined with deep neural networks for function approximation.
Our base agent is R2D2 \citep{kapturowski2018recurrent} and where relevant, we use the same network architecture and hyperparameters from \citet{kapturowski2018recurrent}.
Our policy gradient algorithms are based on PPO \citep{schulman2017proximal}.
For the PPO loss, we use default hyperparameters reported in \citet{schulman2017proximal}: clipping factor $\epsilon=0.2$, shared parameters for policy and value networks with value loss weight $c=1.0$, no entropy bonus, and we ignore the GAE parameter $\lambda$ since our algorithms are optimized myopically.
For PPO, we additionally add to the network a linear layer to predict the value baseline from the final layer of activations.
All other architecture choices and hyperparameters are shared for both Q-learning and policy gradient algorithms.

In both cases, the required change to the learning algorithm to implement decoupled approval is simple -- we independently sample both a query action, provided to the feedback function through a query register, and an action to take directly in the environment, then optimize the policy for approval with discount $\gamma=0$.
The approval signal is defined using a Q-network from the Non-embedded agent in \S\ref{subsec:comparisons}.
This Q-network is pre-trained on the non-embedded version of each task (with oracle access to the true rewards).
For query $k$, we provide feedback $\delta(s, k) = Q(s, k) - \sum_k Q(s, k) / |\mcA|$.\footnote{
In real-world tasks without a programatically defined reward or approval function, this approval signal would be provided by human evaluation instead.
\S\ref{subsec:practical_consid} elaborates on the strengths and weaknesses of actual human feedback compared to our simulated approver.}

We train all agents on 2 TPUv3 cores with 80 distributed actors, and report performance after 8 hours, which corresponds to roughly 100 million total environment steps (though many experiments saturate performance in much less time), except for Standard RL agents which we run for 24 hours.
Our emphasis in this work is on agent incentives, and as such, we only study behavior in the limit.
However, we also note the importance of simultaneously studying sample-efficient learning with approval-based algorithms generally, when using limited human feedback, e.g., \citet{warnell2018deep} and \citet{arumugam2019deep}.

\subsection{Evaluation}
In our experiments, we focus on three main questions:
\begin{enumerate}[label=\textbf{Q{\arabic*}}., noitemsep,topsep=0pt]
\item Is tampering a practical problem for deep RL agents trained on embedded feedback?
\item Does decoupled approval address the problem?
\item How do different algorithms compare on asymptotic task performance?
\end{enumerate}

We measure task performance in the standard way: each environment defines a reward function, and agents are evaluated according to expected total return.
Thus, although algorithms may differ in feedback function (rewards, approval, etc.), all are evaluated on the same metric: total return.

We measure tampering with consideration of how feedback can become corrupted.
We refer to the output of the feedback function, used to set the value of the feedback register, as the \emph{true feedback}, and the value of the feedback register when observed by the agent as the \emph{observed feedback}.
When the agent does not disrupt the normal physical operation of the feedback register, the difference between true and observed feedback is 0 (except for small variations due to noise in the physics simulation).
Thus, we quantify tampering as the discrepancy between true and observed feedback.
Specifically, if the difference exceeds a fixed threshold of $10.0$ at any point within an episode, we classify that episode as one in which tampering occurred.

\subsection{Comparisons}
\label{subsec:comparisons}

We compare decoupled approval agents to several other deep RL algorithms.
Here we describe all our agents:

\begin{description}
\item[Non-embedded]
  This is the only agent trained with oracle access to true rewards, rather than observed rewards from registers.
  The Q-network from this agent is used to define the approval signal for approval-based algorithms.
  Thus, the tampering rates and task performance achieved by the Non-embedded agent represent an ideal to aim for with the approval-based algorithms.
\item[Standard RL]
  This agent (also based on R2D2 \citep{kapturowski2018recurrent}) uses the true reward function as the feedback function $\delta$, optimizing discounted returns as observed through registers (and hence potentially corrupted).
\item[DA-QL / DA-PG]
  These agents use approval feedback, as described above (\S\ref{subsec:exp_design}).
\item[\Approvalrl{} QL / \Approvalrl{} PG]
  The \Approvalrl{} agents \citep{knox2008tamer,knox2009tamer} are otherwise equivalent to their Decoupled Approval counterparts, but always query feedback for their current action, rather than sampling a separate query action.
  We provide this baseline since it is the most direct instantiation of RL from approval -- treat approval as reward, and optimize myopically with RL as usual.
  This allows us to analyze whether decoupling is important for avoiding tampering incentives, or whether myopic optimization alone does most of the work.
  We caution against over-extrapolation: as noted in \S\ref{sec:myopic_sufficiency}, the effect of myopic optimization is likely task-dependent, and may be either sufficient or insufficient depending on the scope of individual actions.
\item[Tampering Detector]
  We run Standard RL and augment the true reward function with a penalty term designed to discourage tampering.
  For the penalty term, we first train a neural network to detect tampering, using a labeled dataset of tampering and non-tampering trajectories generated with Standard RL and Non-embedded agents.
  We use the predicted tampering probability as a weighted negative penalty in the reward function.
  Further, to prevent the agent from continuing to tamper and overriding this penalty, we terminate the episode whenever the probability exceeds $0.5$.
  We validate the accuracy of this tampering detector on a held-out Standard RL agent.

\item[Offline Imitation]
  This agent is trained with behavioral cloning on a cross-entropy loss.
  Specifically, during training time, all actions are drawn from a demonstrator policy, while for evaluation, we use the agent's own policy.
  The learner optimizes the cross-entropy loss $L(\theta) = \hat{\ev}_t\br{-\log \pi_\theta(a_t \mid s_t)}$, where $\hat{\ev}_t$ denotes the empirical expectation over states and actions when following the demonstrator policy.
  When using embedded feedback, the feedback function outputs the demonstrator action, encoded as an integer.
\item[DAgger]
  The DAgger agent is otherwise equivalent to Offline Imitation, but uses the DAgger algorithm to address issues caused by distribution shift between demonstrator and agent trajectories \citep{ross2011reduction}.
  During training, we sample actions according to the agent policy with probability $0.9$, and according to the demonstrator policy with remaining probability $0.1$.
\end{description}

\subsection{Results}
\label{subsec:tampering_results}

\begin{table}[ht]
\centering
\begin{threeparttable}
  \begin{tabular}{p{0.4\columnwidth}cc}
\toprule
    \textbf{Algorithm} &  \thead{\textbf{Tampering} \\ \textbf{(\% Episodes)}} & \thead{\textbf{Return} \\ \textbf{(Mean)}}\tabularnewline
\midrule
    Standard RL & 78\% & \hp1.0 \tabularnewline
    Tamp.\ Detector & 84\% & \hp1.2 \tabularnewline
    Non-embedded $^*$ & \hp6\% & 14.8 \tabularnewline
    Tamp.\ Detector (non-embedded) $^*$ & \hp2\% & 14.4 \tabularnewline
\midrule
    \Approvalrl{} QL & 26\% & \hp9.5 \tabularnewline
    DA-QL (no IS corr.) & 19\% & 10.2 \tabularnewline
    DA-QL & 12\% & 11.9 \tabularnewline
\midrule
	\Approvalrl{} PG & 18\% & 13.4 \tabularnewline
    DA-PG & \hp6\% & 13.7 \tabularnewline
\midrule
   Offline Imitation $^*$ & \hp8\% & 14.0 \tabularnewline
   DAgger $^*$ & \hp8\% & 14.0 \tabularnewline
\bottomrule
\end{tabular}
\vspace{1mm}
\begin{tablenotes}
\small
\item $^*$ Trained purely with non-embedded feedback
\end{tablenotes}
\end{threeparttable}
  \caption{\textbf{Unlock Door results.}
Standard RL agents naturally result in tampering, while
decoupled approval algorithms improve over their non-decoupled approval counterparts both in terms of tampering rates and task performance.
}
\label{tab:tampering_results}
\end{table}

Table \ref{tab:tampering_results} summarizes asymptotic task performance and percentage of episodes containing tampering for our experiments in the Unlock Door task.
Below, we summarize our key observations for each agent.
For qualitative results, we encourage interested readers to view our accompanying agent videos at \href{\videolink}{\videolink}.

\sparagraph{Standard RL}
As hypothesized, the Standard RL agent converges to a tampering solution -- the agent tampers in 78\% of episodes overall.

\sparagraph{\Approvalrl{} RL}
The use of myopic optimization in \Approvalrl{} RL significantly reduces tampering, since the agent lacks incentives to create future opportunities to tamper.
We observe that the agent tampers if and only if it encounters registers (i.e., tampering opportunities) along the trajectory of solving the intended task.
In Figure \ref{subfig:myopic}, the agent initially gathers apples (the intended behavior) until it encounters a register.
It then throws this register so as to achieve high observed reward, before proceeding to solve the intended task.
This demonstrates characteristic behavior of myopically optimized agents -- they take opportunities for tampering when present, but they do not select actions to create future opportunities to tamper.
In \S\ref{subsec:other_approaches}, we discuss the implications for applying \Approvalrl{} RL in the real world.

\sparagraph{Decoupled Approval}
Consistent with our theoretical results, both DA-PG and DA-QL result in significantly less tampering than their non-decoupled approval counterparts.
Further, DA-QL tampers more often when the importance sampling correction is ablated.
However, while the tampering rate for DA-PG exactly matches the Non-embedded agent (6\%), for DA-QL it is slightly higher -- one possible explanation is that nonlinear function approximation reduces the effectiveness of the importance sampling correction, though other explanations are also possible.

Notably, our experiments use a \textit{tampering-agnostic} evaluator, which makes no effort to provide negative reward when the agent tampers.
As discussed in \S\ref{subsec:adpg}, this resembles real-world applications where it is difficult to detect tampering during training.
In such cases, we cannot expect decoupled approval agents to actively avoid feedback corruptions, but it is desirable that they do not pursue corruptions.
Our empirical results generally support this view: while tampering rates for DA-PG and DA-QL are not zero, they are similar to tampering rates for the Non-embedded agent (6\% and 12\% vs.\ 6\%). \\

\noindent
\emph{Task performance.}
Decoupled approval algorithms achieve competitive asymptotic performance compared to the Non-embedded RL agent, despite having the option to tamper.
In particular, they achieve higher returns than both their non-decoupled counterparts and the standard RL agent, which exclusively tampers at the expense of task performance.

\sparagraph{Tampering Detector}
First, to validate the accuracy of our learned tampering detector, we apply it to a pretrained held-out Standard RL agent.
For this agent, the detector fails to fire in 0\% of episodes in which tampering occurs, and fires (resulting in episode termination) in 98\% of all episodes.
We further validate that the detector provides a useful training signal, by adding the detector to the Non-embedded agent.
This decreases tampering from 6\% to 2\%, though does not eliminate tampering completely, while the agent still achieves high return.
For the Non-embedded agent, 0\% of episodes are terminated early by the tampering detector.

In the embedded setting, however, the Tampering Detector agent eventually learns to bypass the detector (though this happens later in training than for the Standard RL agent).
Qualitatively, the agent spins between the two Register blocks, and moves each block a little bit at a time, as opposed to the non-shaped Standard RL agent, which simply moves a single block to the desired location, without alternating between the two blocks (see video supplement).
In contrast to the held-out Standard RL agent, which also attempts to tamper, the Tampering Detector agent bypasses the detector, tampering without triggering an episode termination in 46\% of episodes.

\sparagraph{Imitation Learning}
Both imitation learning variants, offline imitation and DAgger, achieve good results in the non-embedded setting.
With embedded feedback, we observed significantly worse performance, which we believe is largely due to specifics in our current register implementations (see Appendix \ref{app:imitation_exps}).

\sparagraph{Other results}
We include additional experiments on all our main agents for the Seek-Avoid task in Appendix \ref{app:seekavoid} and on procedurally generated tabular CFMDPs in Appendix \ref{app:tabular_cfmdp}.

\begin{figure}[ht]
\centering
\begin{subfigure}[b]{0.48\textwidth}
\includegraphics[width=\textwidth]{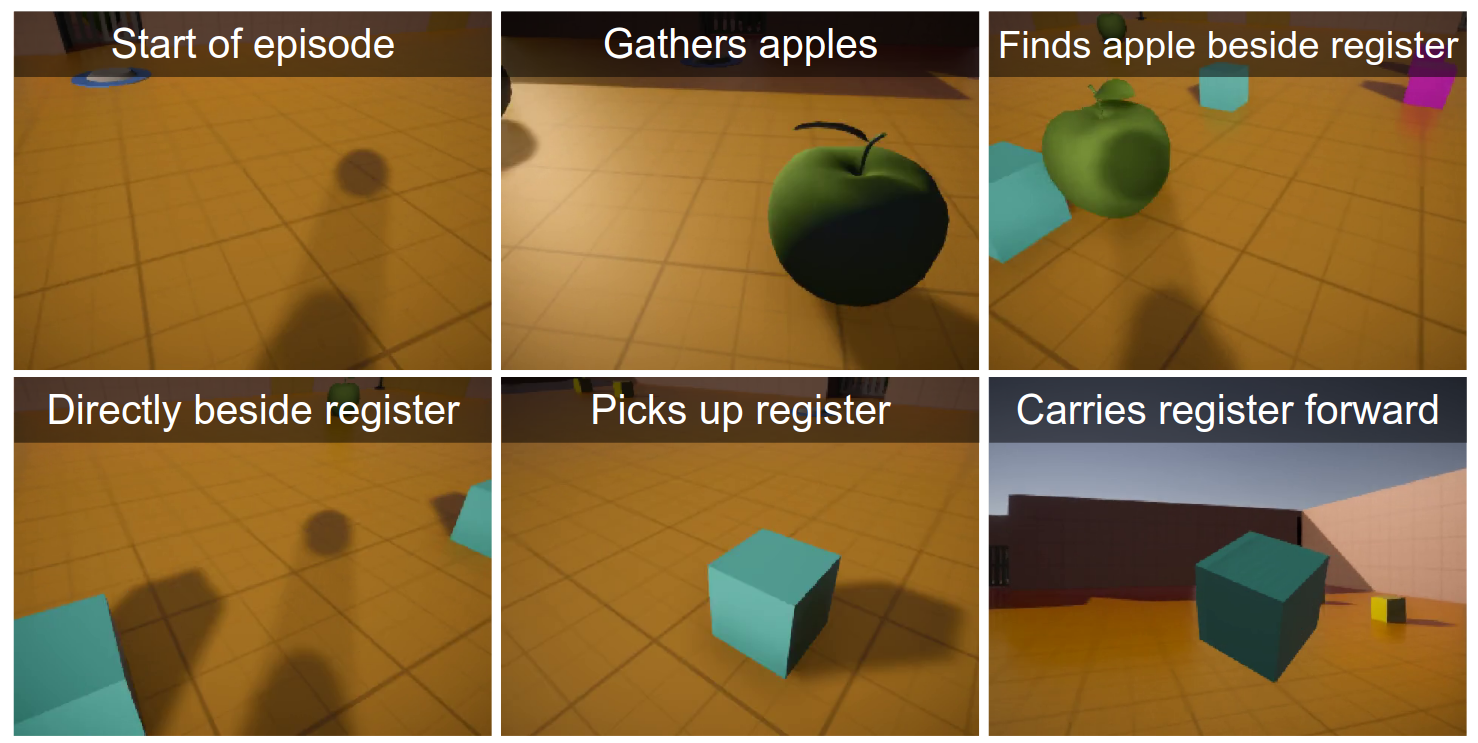}
\caption{\Approvalrl{} RL Agent}
\label{subfig:myopic}
\end{subfigure}\hfill
\begin{subfigure}[b]{0.48\textwidth}
\includegraphics[width=\textwidth]{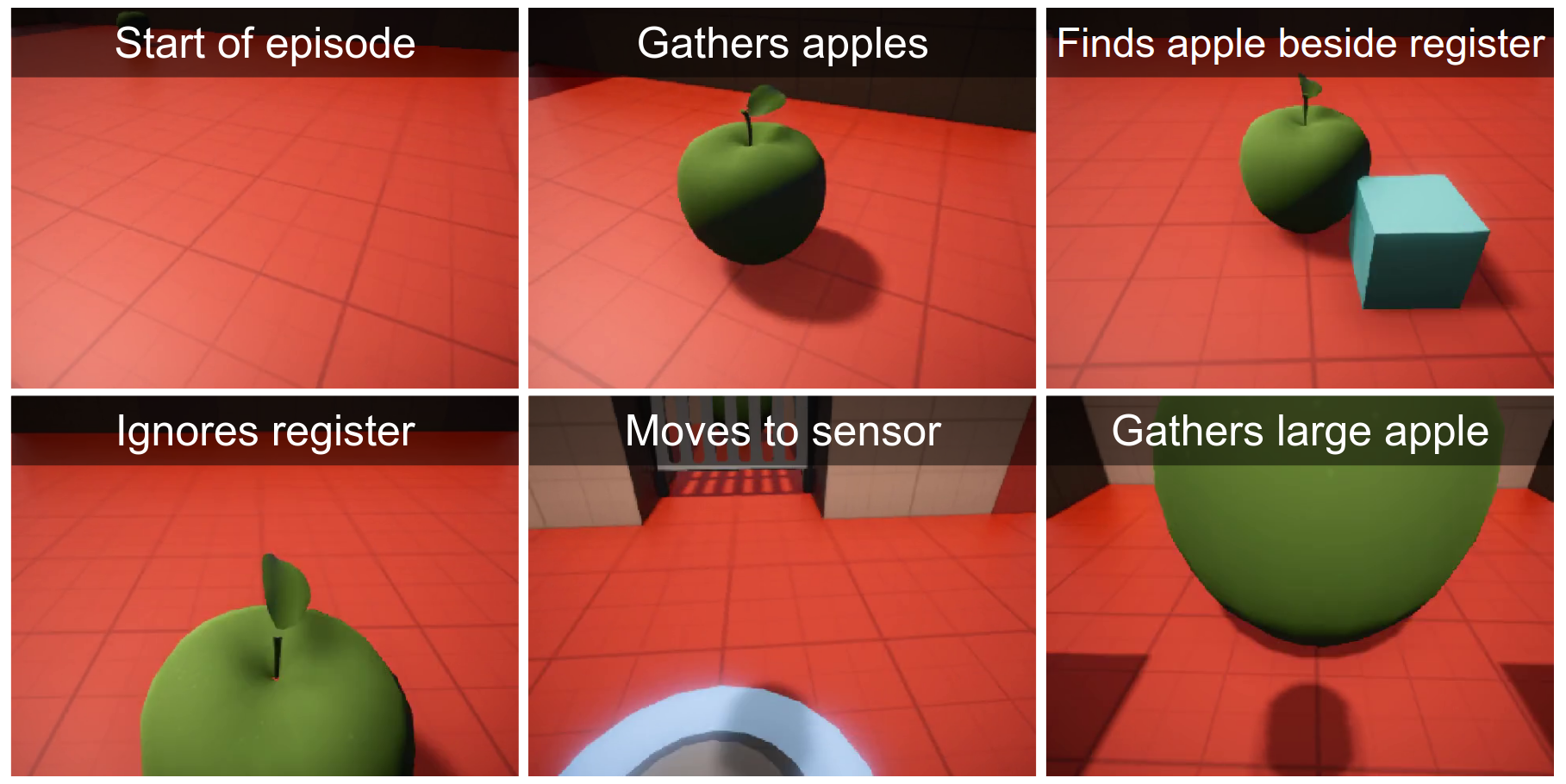}
\caption{DA-QL Agent}
\label{subfig:da}
\end{subfigure}
\caption{\textbf{Qualitative agent behaviors.}
The \Approvalrl{} RL agent mostly gathers apples as intended, but tampers opportunistically when available actions can cause immediate corruptions to the feedback register, such as by spinning and carrying the register.
This typically occurs when the agent's typical apple-gathering trajectory brings it close to a register.
In contrast, the DA-QL agent ignores the presence of registers and focuses on the intended task, despite also having opportunities to take actions causing feedback corruption.
Videos for all our agents are available in our video supplement at \href{\videolink}{\videolink}.
}
\label{fig:agent_sequences}
\end{figure}

\subsection{Robustness Checks}

All reported results are the average of 2 seeds.
We measure statistics using a fixed model at convergence, with averages taken over 5000 episodes.
Overall, we observe fairly little variation across results, since we wait until convergence before evaluating agents.
Overall, the deviation in total return across seeds is $< 1$ for all agents, except \Approvalrl{} QL on Seek-Avoid.
For tampering, variation across seeds is less than $3\%$, except for the standard RL agents, for which we see variation of $20\%$ across 3 seeds.
We believe this variation arises because the scale of returns for Standard RL is considerably larger than standard, for example by roughly a factor of 300 relative to our non-embedded setup, which causes instability in training.

\section{Related Work}
\label{sec:related_work}
We provide a high-level overview of related work, while deferring discussion of some more detailed comparisons to \S\ref{sec:discussion}.

\paragraph{The tampering problem}
A large literature discusses accurately modeling the tampering problem
~\citep{ring2011delusion, soares2015corrigibility, Bostrom2016Superintelligence, Armstrong2017Counterfactual, everitt2018thesis, Demski2019embedded}, as well as the closely related \emph{reward gaming} problem \citep{Amodei2016concrete,leike2018scalable,milli2020optimizing}.
We contextualize our work within the CFMDP formalism and REALab platform developed by \citet{kumar2020realab}, to which we also refer readers for a broader survey on the general tampering problem.

\paragraph{Approval and decoupling}
Value advice algorithms such as TAMER \citep{knox2008tamer} and COACH \citep{celemin2015coach} first demonstrated the practical feasibility of using approval as feedback.
Subsequent deep RL variations \citep{warnell2018deep,arumugam2019deep} building on these algorithms demonstrated impressive performance with low sample complexity.
This paper builds on approval-direction \citep{Christiano2014approval,Christiano2015Approval,Christiano2016Approval} in two ways: by deriving precise claims about the incentives of algorithms formalized within CFMDPs, and by investigating how these effects are reflected empirically.

Counterfactual Oracles \citep{Armstrong2017Counterfactual} use a similar decoupling approach to design question-answering systems where any system output observed by a human supervisor is never used as training feedback, and vice versa.
Decoupling is also a key idea in Current Reward Function Optimization (CRFO \citep{everitt2018thesis, dewey2011learning}).
CRFO uses observational data to fit both a transition and reward function, and then optimizes the policy solely against these learned models.
The idea is that policy rollouts within the simulated transition model will not be able to directly influence the world outside this simulation, and thus cannot directly tamper with the learned reward model.

\paragraph{Other solutions to the tampering problem}
Several approaches to the tampering problem involve inferring (a distribution over) the latent underlying reward function from observed data.
For example, (Cooperative) Inverse Reinforcement Learning \citep{ng2000algorithms,abbeel2004apprenticeship,HadfieldMenell2016CIRL} involves inferring the latent reward parameter from demonstrations, typically under an assumption of approximately-optimal demonstrations.
Inverse Reward Design \citep{hadfield2017inverse} similarly uses the observed rewards as evidence about the true latent reward function.
Importantly, these algorithms provide agents with uncertainty over designer preferences, which encourages them to maintain reliable feedback mechanisms and learn more about designer preferences.
An advantage of decoupled approval is that it can be straightforwardly integrated with standard deep RL techniques, without requiring accurate uncertainty estimates in the learned posterior, which is a key difficulty and active area of research for Bayesian neural networks \citep{graves2011practical,yao2019quality}.
In general, our work is the the first to our knowledge which empirically studies tampering incentives in the context of neural networks.

Another class of approaches involve explicitly identifying tampered feedback to avoid training on such data \citep{Everitt2017CRMDP,Mancuso2019Spiky,leike2018scalable}.
These either rely on costly checking mechanisms, or preserve tampering incentives if there are any loopholes in the tampering detector.
An advantage is that these approaches can actively discourage tampering, rather than merely being tampering-agnostic, and can thus complement other solutions.
We discuss detecting tampering further in \S\ref{subsec:other_approaches}.

Turning to our analysis, our work is the first to our knowledge which analyzes tampering incentives while treating the learning algorithm as a stochastic process.
The majority of prior work on tampering incentives focuses on utility functions: how will agents behave, assuming they select actions according to the argmax over some utility function \citep{soares2015corrigibility,Armstrong2017Counterfactual,Everitt2019tampering}?
Causal Influence Diagrams (CIDs) express learning algorithms using causal graphs \citep{Everitt2019tampering,everitt2019modeling,carey2020incentives}, enabling use of graphical algorithms to determine the presence or absence of tampering incentives.
We discuss this point further in \S\ref{subsec:disc_analysis_approaches}.

\section{Discussion}
\label{sec:discussion}

\subsection{Analysis of Algorithms vs.\ Analysis of Objectives}
\label{subsec:disc_analysis_approaches}
Much prior work assumes the agent's utility function is known for analysis.
This raises an important question we address: given a mechanistic description of a learning algorithm, such as code or pseudocode, how can we determine which function this algorithm optimizes (if any)?
Determining a utility function can be subtle, as we illustrate with the DA-QL example in Appendix \ref{app:importance-sampling}.
The key point is that although DA-QL with and without the importance sampling correction are trained with the same feedback function, they converge to different solutions.
Thus, determining the utility function implied by an algorithm must necessarily look beyond expected returns alone.

Further, many deep RL systems contain algorithmic ideas which can complicate determining a utility function -- for example, varying loss reweightings \citep{schaul2015prioritized}, meta-learning or optimizers at multiple timescales optimizing different metrics \citep{finn2017model,jaderberg2017population}, multiple loss functions \citep{ho2016generative}, and auxiliary tasks \citep{jaderberg2016reinforcement}.
\citet{krueger2019misleading} discuss a similar question specifically in the context of multi-scale optimizations, such as in \citet{jaderberg2017population}.

These factors motivate our decision to treat learning algorithms as stochastic processes -- without assuming a utility function -- and to then derive convergence and update incentives from this process directly.
We note that it may be desirable to go further still and analyze models themselves produced by algorithms \citep{hubinger2019risks}, though the feasibility of such analysis remains an important open problem \citep{olah2018building,olah2020zoom}.
Overall, we believe high-level and low-level analysis techniques are complementary -- by abstracting away details, high-level approaches can suggest useful algorithmic design principles,
while low-level analyses can validate the assumptions high-level approaches require.

\subsection{Practical Considerations}
\label{subsec:practical_consid}

\paragraph{Sample complexity}
A limitation of our experiments is their large sample complexities, which means our efficiency comparisons between algorithms may not hold within many regimes of practical interest.
This is a common challenge for deep RL \citep{henderson2018deep}, but is particularly critical for approval-based algorithms due to their reliance on human feedback, which is typically costly.
Semi-supervised RL \citep{finn2016generalizing} offers a promising solution to leveraging unsupervised experience in learning, and seems particularly relevant for approval-based approaches.
Indeed, recent approval-based algorithms such as \citet{warnell2018deep} and \citet{arumugam2019deep} leverage unsupervised experience through an auxiliary reconstruction loss and demonstrate impressive sample complexity by learning Atari Bowling and simple 3D Minecraft levels with less than 15 minutes of human feedback.
These results suggest the possibility for approval-based techniques to learn significantly faster than purely reward-driven techniques, due to stronger supervision.

\paragraph{Designing useful approval signals}
While not addressed in this work, choices in agent design may have a large impact on the ability to elicit useful approval feedback from humans.
For instance, in training a robotic system, supervising low-level actuator commands would be less natural than supervising actions specifying goal locations \citep{nachum2018data}.
Approval-based algorithms may benefit significantly from improvements in structured action spaces \citep{jiang2019language,barto2003recent} and preference elicitation methods \citep{gajos2005preference}.
However, approval feedback may be harder to provide when decoupled, since the supervisor needs to model unobserved state transitions in order to give feedback on queried actions that are not taken.
An important question is how improvements in ML can be effectively leveraged to design better approval signals.
Amplification \citep{christiano2018supervising} proposes using the learned agent itself to provide assistance to the supervisor.
In debate \citep{irving2018ai}, the task itself is structured to minimize the difficulty for the supervisor.

\paragraph{Allocation of errors}

In general, agents will make some errors, even with perfect feedback.
Thus, we would like the objective to prioritize the errors with the greatest impact on the policy's expected return, $Q^\pi$.
An advantage of reward-based policy gradient methods is that policies are optimized directly for $Q^\pi$.
Behavioral cloning fares poorly in this regard, since the log-loss objective does not differentiate between important actions with high influence on $Q^\pi$ and unimportant actions.
Relative to behavioral cloning, decoupled approval allows the supervisor to vary the scale of approval across different steps.
For example, if supervisors provide near-constant approval when evaluating unimportant actions, the objective will prioritize important actions.

We note that the approval signal in our experiments approximates $Q^*$, as in the value advice setting.
However, since actions should be penalized based on $Q^\pi$ rather than $Q^*$, %
the approval signal should ideally approximate $Q^{\pi_t}$ instead.
This presents practical difficulties, as it would require the approval signal adapt over time to $\pi_t$, and may be an unnatural form of supervision for humans to provide.
We leave further empirical and theoretical investigation into the importance of this deviation to future work.

\subsection{Assumptions}

\paragraph{Accurate approval feedback}
This work assumes access to accurate approval feedback.
The question of how to implement accurate feedback (\S\ref{subsec:practical_consid}) is important for future work.

\paragraph{\UniformCorruptions}
We believe that the assumption of \uniformCorruptions{}, as stated, is too strong to hold in many real-world settings.
This assumption is violated when there are \emph{self-aggrandizing actions}, which can increase observed feedback specifically for queries about the same action.
It may be possible to achieve guarantees even under significantly weaker assumptions by demonstrating that aligned local update incentives are preserved with high probability, so long as the probability of selecting self-aggrandizing actions is sufficiently small at each step, similar to Propositions \ref{obs:adpg_unbiased} and \ref{prop:adql_local_incentives}.
We suspect such considerations may allow decoupled approval to avoid tampering even in settings where the assumption of \uniformCorruptions{} does not hold (including the REALab environments we study), but leave more rigorous investigation to future work.

\paragraph{Generalization}
An important question is whether agents can still have tampering incentives, even in the presence of well-designed checks which can prevent tampering during training.
Answering this requires understanding the role of generalization --
a natural starting point would be to analyze the behavior of decoupled approval beyond the tabular setting, such as when used with linear function approximation.

\subsection{Other Approaches to Avoiding Tampering}
\label{sec:myopic_sufficiency}
\label{subsec:other_approaches}

\paragraph{Myopic optimization}
\Approvalrl{} RL agents are incentivized to tamper in states where tampering actions are available which can produce large positive corruptions.
However, due to myopic optimization, \approvalrl{} RL agents are not incentivized to visit such a state, except insofar as there is high approval along paths towards that state.
Therefore, as long as all opportunities to tamper are far from trajectories that are followed by the optimal policy, myopic optimization may be sufficient, without decoupling, to avoid learning to tamper.
The question of whether or not tampering opportunities are available along optimal trajectories will depend on task-specific details.

We believe that applying \approvalrl{} RL in the real world would likely involve optimization over ``high-level'' actions.
For example, we might consider cases such as high-level actions in a hierarchical RL setup \citep{sutton1999between,barto2003recent}, or an action space of instructions issued to a lower-level controller that has been optimized non-myopically in a narrow domain \citep{reed2015neural, merel2018neural}.
Such approaches seem particularly advantageous for approval-based methods, since they both make providing approval feedback more natural, and provide the agent greater flexibility, when compared to low-level action spaces.
These approaches also motivate decoupling: in the space of high-level actions, it may be possible for a single action in an otherwise near-optimal trajectory to produce tampering, by initiating an extended sequence of low-level actions.

\paragraph{Detecting tampering}
Detecting and penalizing tampering is a natural approach for handling the potential for corrupt feedback: run RL as normal, while ensuring that every training-time instance of detected tampering is associated with negative reward.
Indeed, for a tampering detector with no false negatives, tampering detection should be sufficient to avoid tampering.
However, our experiments illustrate a central difficulty --
if the agent finds any way to bypass the detector, it will be incentivized to exploit this loophole.
We believe tampering detection schemes are best combined with algorithms that induce aligned incentives: the penalty discourages tampering, while incentives ensure that undetected tampering actions are at least not incentivized.

Finally, we note that for real-world applications, a simple option is to use a tampering detector to decide whether to deploy the system.
Thus, as long as it is possible to detect non-zero tampering levels, as is the case here, the algorithm designers can decide not to deploy any models that tamper.
A more challenging case is if tampering can be detected early during training, but the agent eventually learns a tampering policy that can evade detection.
From the agent designers' perspective, it may appear that tampering detection has removed tampering behaviors, when in fact, the agent has learned a policy to bypass the detector.

\paragraph{Imitation learning}
An important question is when imitation-based or approval-based algorithms are preferable. %
Behavioral cloning can provide a lower variance training signal than policy gradient estimators, and also aid in hard exploration problems \citep{paine2019making,vinyals2019grandmaster}.
On the other hand, approval-based objectives can prioritize important actions, as discussed in \S\ref{subsec:practical_consid}.
Ultimately, we believe both approaches can yield good results, and
 note the two objectives can also be combined.

An important motivation for this work's focus on approval-based algorithms is pedagogical: we show that agents can avoid tampering incentives, even while maximizing a reward function in a corrupt feedback setting.
However, we hope future work will also analyze tampering incentives for imitation algorithms.
A particularly interesting setting is when feedback is gathered online, and the agent's actions can influence demonstrator behavior.

\section{Conclusion}

In this paper, we showed both theoretically and empirically that standard RL algorithms have incentives to tamper with the user's feedback.
This tampering incentive remains even for algorithms using approval feedback.
However, with a simple change, decoupled approval algorithms avoid tampering incentives, both in practice, and under the additive query-independent corruptions assumption.
Decoupled approval is naturally compatible with standard deep RL techniques, and when applied to a 3D environment, demonstrates strong performance without tampering.

\paragraph{Limitations and future work}
Our results raise many further questions, which we discuss in \S\ref{sec:discussion}.
On the theory side, we use the strong assumption of \uniformCorruptions, and restrict attention to the tabular setting.
It would be good to explore weaker assumptions on the corruption function and the effects of generalisation by the learner on our incentives analysis.
On the experimental side, our algorithms are data inefficient, and improving this is especially important given the costly nature of human feedback.
Finally, there is scope for exploring different forms of feedback, requiring less of the feedback provider while still allowing algorithms with aligned tampering incentives.

\section*{Acknowledgements}
We would like to especially thank Paul Christiano, Evan Hubinger, and Rohin Shah for extensive discussions in developing and refining these ideas.
We are also grateful to Orlagh Burns, Valentin Dalibard, Sarah Ellis, Adam Gleave, Zac Kenton, Pushmeet Kohli, Eric Langlois, Jan Leike, Vladimir Mikulik, Chongli Qin, Matthew Rahtz, and Alex Zhu
for many helpful discussions throughout the course of this work.

\bibliographystyle{plainnat}
\bibliography{refs}
\newpage
\appendix
\setcounter{theorem}{0}
\renewcommand{\thetheorem}{\Alph{section}\arabic{theorem}}
\setcounter{figure}{0}
\renewcommand\thefigure{\thesection.\arabic{figure}}
\section{Proofs for Non-decoupled \Approvalrl{} RL}\label{app:myopic-tampering}

\begin{proposition}[More details for Remark \ref{remark:myopic}]
Standard \approvalrl{} RL algorithms optimize a corrupted signal.
\end{proposition}

\begin{proof}
Let $a^*(s) = \arg\max_a \delta(s,a)$ be the optimal action for the feedback function in state $s$. Assume there is a state $s$ where some action $a'$ has a different next state distribution than the optimal action $a^*(s)$.
Choose a state $s'$ such that $D = f(s' \mid s,a') - f(s'\mid s,a^*(s)) > 0$. Such a state must exist by assumption that $a'$ and $a^*(s)$ have different next state distributions. Assign $c_{s'} = L$ where $L$ is a large number, and $c_x = 0$ for $x \not= s'$.
Then
\begin{align*}
Q(s,a') - Q(s,a^*(s))
&= [ f(s'\mid s,a') L + \delta(s,a') ] - [ f(s'\mid s,a^*(s)) L + \delta(s,a^*(s)) ] \\
&= D L - [ \delta(s,a^*(s)) - \delta(s,a') ].
\end{align*}
This is positive if $L > [ \delta(s,a^*(s)) - \delta(s,a') ] D^{-1}$. Thus, for large enough $L$, $a^*(s)$ is not the optimal action in state $s$ for the corrupted feedback function.
\end{proof}

\section{Proofs for DA-PG}\label{app:adpg}

\begin{corollary}[Corollary \ref{cor:da-pg}]
Consider policies parameterized by a mixture parameter $z$ over two experts $\pi_1$ and $\pi_2$: $\pi_z(a\mid s) = \sigma(z)\pi_1(a\mid s) + (1 - \sigma(z))\pi_2(a\mid s)$, where $\sigma(z) = (1+\exp(-x))^{-1}$ is the sigmoid function.
DA-PG applied to this policy parameterization induces aligned update incentives.
That is,
  $\overline{R}(\pi_{\overline{z_{t+1}}}) \geq \overline{R}(\pi_{z_t})$, where $\overline{R}(\pi) = \ev_{a\sim\pi(\cdot\mid s)}\br{\delta(s, a)}$ and $\overline{z_{t+1}} = \ev\br{z_{t+1}\mid z_t, s}$.
\end{corollary}
\begin{proof}
Following the DA-PG update given in Algorithm \ref{alg:adpg} by $\theta_{t+1} = \theta_t + \alpha \tilde d \nabla_\theta \log \pisk$, the expected update for $z$ is
  $$\overline{z_{t+1}} = \ev\br{z_{t+1}\mid{z_t},s}=z_t+\alpha_t\underbrace{\ev\left[c(s',k,\delta(s,k))\frac{d}{dz_t}\log\pi_{z_t}(k\mid s)\right]}_g.$$

Using Proposition \ref{obs:adpg_unbiased}, and noting that $a$ and $k$ are sampled independently from the same policy, we analyze the update directly:
\begin{align*}
g =&\; \ev_{k\sim\pi_{z_t}(\cdot\mid s)}\left[c(s',k,\delta(s,k))\frac{d}{dz_t}\log\pi_{z_t}(k\mid s)\right]\\
  =&\; \ev_{a\sim\pi_{z_t}(\cdot\mid s)}\left[\delta(s,a)\frac{d}{dz_t}\log\pi_{z_t}(a\mid s)\right] \quad\quad \text{[apply Proposition \ref{obs:adpg_unbiased}]}\\
  =&\; \ev_{k\sim\pi_{z_t}(\cdot\mid s)}\left[\delta(s,k)\frac{d}{dz_t}\log\pi_{z_t}(k\mid s)\right] \quad\quad \text{[$a$ and $k$ are from the same policy]}\\
  =&\; \sum_{k\in\mcK} \delta(s, k) \frac{d}{dz_t}\pi_{z_t}(k\mid s) \\
  =&\; \frac{d}{dz_t} \sum_{k\in\mcK} \delta(s, k) \left[ \sigma(z_t)\pi_1(k\mid s) + (1 - \sigma(z_t))\pi_2(k\mid s) \right] \\
  =&\; \frac{d}{dz_t} \left[ \sigma(z_t) \sum_{k\in\mcK} \delta(s, k) \pi_1(k\mid s) + (1 - \sigma(z_t))\sum_{k\in\mcK} \delta(s, k) \pi_2(k\mid s) \right] \\
  =&\; \frac{d}{dz_t} \left[ \sigma(z_t) \overline{R}(\pi_1) + (1 - \sigma(z_t)) \overline{R}(\pi_2) \right] \\
  =&\; \sigma'(z_t) \left[ \overline{R}(\pi_1) - \overline{R}(\pi_2)\right]
\end{align*}
  Note that $\sigma'(x) = \exp(-x) / (\exp(-x) + 1)^2$ is always positive and $\alpha_t > 0$. Thus, when $\overline{R}(\pi_1) \geq \overline{R}(\pi_2)$, we have $g \geq 0$ and so $\overline{z_{t+1}} \geq z_t$ and $\overline{R}(\pi_{\overline{z_{t+1}}}) \geq \overline{R}(\pi_{z_t})$.
  Similarly when $\overline{R}(\pi_1) < \overline{R}(\pi_2)$, we have $\overline{z_{t+1}} < z_t$ and hence $\overline{R}(\pi_{z_t}) < \overline{R}(\pi_{\overline{z_{t+1}}})$.
\end{proof}

\section{Proofs for DA-QL}\label{app:adql}

\subsection{Proof of Proposition \ref{prop:adql_local_incentives} (aligned local update incentives)}\label{app:adql-local}

\begin{proposition}[Proposition \ref{prop:adql_local_incentives}]
In DA-QL under \uniformCorruptions, for all $s, k, t$, we have
$$\ev\br{\xi_t(k) \mid H_t} =   h_1 (\delta(s,k) - Q_t(s,k)) + h_2$$
where $\xi_t(k) = Q_{t+1}(s,k) - Q_t(s,k)$, $H_t$ is the history up to time $t$ (including the current state $s$ but not $a$ or $k$), and $h_1, h_2$ are functions of $H_t$.
In particular, whenever $\delta(s, k_1) > \delta(s, k_2)$, at least one of the following holds:
\begin{equation*}
\begin{split}
\ev\br{\xi_t(k_1)\mid H_t} &> \ev\br{\xi_t(k_2)\mid H_t} \;\;\;\; or \\
Q_t(s, k_1) &> Q_t(s, k_2)
\end{split}
\end{equation*}
\end{proposition}
\begin{proof}
Let $\beta_t(s) = \ev[\alpha_t(s,k) \mid H_t]$. Then we have
\begin{align*}
\ev\br{\xi_t(k) \mid H_t}
&= \ev[Q_{t+1}(s,k)\mid H_t] - Q_t(s,k) \\
&= \ev[(1-\alpha_t(s,k))Q_t(s,k)+\alpha_t(s,k)(\delta(s,k)+c_{s'}) \mid H_t] - Q_t(s,k) \\
&= (1- \beta_t(s)) Q_t(s,k) + \beta_t(s) (\delta(s,k)+\ev[c_{s'} \mid H_t]) - Q_t(s,k) \\
&= \beta_t(s) (\delta(s,k) - Q_t(s,k))  + \ev[c_{s'} \mid H_t])
\end{align*}
  where the third line relies on the fact that $\alpha_t(s,k)$ is independent of $c_{s'}$, given $H_t$, because the corruption is independent of the query action under \uniformCorruptions.
Thus, the required statement holds with $h_1 = \beta_t(s)>0$ and $h_2 = \beta_t(s) \ev[c_{s'}\mid H_t]$.
Then,
$$\ev[\xi_t(k_1)-\xi_t(k_2) \mid H_t] = h_1 ( (\delta(s,k_1) - \delta(s,k_2)) - ( Q_t(s,k_1) - Q_t(s,k_2)) )$$
which is positive if $Q_t(s, k_1) \leq Q_t(s, k_2)$ and $\delta(s,k_1) > \delta(s,k_2)$, as required.
\end{proof}

\subsection{Proof of Proposition \ref{thm:adql} (convergence incentives)}\label{app:adql-global}

\paragraph{Stochastic process convergence}
The following theorem is the basis for the convergence proof for Q-learning in MDPs~\citep{Jaakkola1993convergence}, and we will also use it for decoupled approval Q-learning.
\begin{theorem}\label{thm:sa-convergence}
Let $\Delta_{t}$, $\beta_t$ and $F_t$, for integers $t\geq{0}$, be sequences of random vectors each with domain $\mathbb{R}^{|\mathcal{X}|}$, where $\mathcal{X}$ is some finite set.
We write $\Delta_t(x)$ for the $x$th component of $\Delta_t$, and similarly for $F$ and $\beta$.
Let $P_t$ be a sequence of increasing $\sigma$-fields such that $\Delta_0$ and $\beta_0$ are $P_0$-measurable, and $\Delta_{t+1}$, $\beta_{t+1}$, and $F_t$ are $P_{t+1}$-measurable for integers $t\geq{0}$.
Suppose the following conditions all hold:
\begin{enumerate}
\item $\Delta_{t+1}(x) = (1 - \beta_t(x)) \Delta_t(x) + \beta_t(x) F_t(x)$, for all $x\in\mathcal{X}, t\geq 0$.
\item For all $x\in\mathcal{X}$, $\sum_{t=0}^\infty \beta_t(x) = \infty$, and $\sum_{t=0}^\infty \beta_t^2(x) < \infty$, both almost surely.
\item $\exists{\gamma\in(0,1)}.\,\forall{t\geq 0}.\,\norm{\ev[F_t \mid P_t]}_\infty \leq \gamma \norm{\Delta_t}_\infty$ (surely).
\item $\exists{C\in\mathbb{R}}.\,\forall{t\geq 0}.\,\Var[F_t \mid P_t] \leq C(1 + \norm{\Delta_t}_\infty)^2$ (surely).
\end{enumerate}
Then $\Delta_t$ converges to $0$ almost surely.
\end{theorem}

\paragraph{Decoupled Approval Q-Learning convergence}
For decoupled approval Q-learning with \uniformCorruptions, we can show that the difference in action value between two actions eventually equals the difference in the (uncorrupted) feedback data for those actions.
This implies that the greedy policy on the learned Q values at convergence maximizes the true feedback.

\begin{theorem}\label{thm:dmql-convergence}[Proposition \ref{thm:adql}]
Let $k, k' \in \mcK$ be two distinct query actions.
For DA-QL in CFMDPs with \uniformCorruptions, the following expression converges almost surely to 0 as $t\to\infty$:
\begin{align*}
\left(Q_t(s,k)-Q_t(s,k')\right)-
\left(\delta(s,k)-\delta(s,k')\right)
\end{align*}
\end{theorem}

\begin{proof}
We define the agent-environment interaction as a discrete-time stochastic process,
with sequences of random variables $S_t$, $K_t$, $A_t$, and $\tilde D_t$, given a behaviour policy $\pi_A$, a query policy $\pi_K$, and a data function $\delta:\mcS\times\mcK\to\mathbb{R}$, as follows.
\begin{align*}
S_0&\sim p(\cdot)\\
K_t&\sim\pi_K(\cdot\mid H_t)\\
A_t&\sim\pi_A(\cdot\mid H_t)\\
\tilde D_{t+1}&=c(S_{t+1}, K_t, \delta(S_t, K_t))\\
S_{t+1}&\sim f(\cdot\mid S_t, A_t)
\end{align*}
Here, the sequence of random variables $H_t$ represent the history up to time step $t$, defined by $H_t=S_0K_0A_0\tilde D_1\dots A_{t-1}\tilde D_tS_{t}$.

As given in Algorithm \ref{alg:adql}, let $s'$ be the next state, and let the learning rate $\alpha_t(s,k)$ be $\lrscale(M_t(s) \pi_K(k|s))^{-1}$ if $S_t=s,K_t=k$ and $0$ otherwise.
We note that the update rule for $Q_t$ can be written equivalently as
\begin{align*}
    Q_{t+1}(s,k)=(1-\alpha_t(s,k))Q_t(s,k)+\alpha_t(s,k)(\delta(s,k)+c_{s'})
\end{align*}
because $\tilde D_{t+1}=\delta(s,k)+c_{s'}$ whenever $\alpha_t(s,k)\neq0$, since the learning rate is zero unless $S_t = s, K_t=k$.
We use this update rule in what follows.

We will instantiate Theorem~\ref{thm:sa-convergence} so that its conclusion matches ours.
Take $\mathcal{X}=\mcS\times\mcK\times\mcK$, which is finite since the underlying sets are finite.
Define $Y_t(s,k)=Q_t(s,k)-\delta(s,k)$.
Now take the following instantiation:
\begin{align*}
\Delta_t(s,k,k')&=Y_t(s,k)-Y_t(s,k')\\
\beta_t(s,k,k')&= \ev[\alpha_t(s,k) \mid H_t]
\end{align*}
Let us abbreviate
$\alpha_t=\alpha_t(s,k)$, $\alpha'_t=\alpha_t(s,k')$, $\beta_t=\beta_t(s,k,k')$, $Y_t=Y_t(s,k)$, $Y'_t=Y_t(s,k')$, and $\Delta_t=\Delta_t(s,k,k')$.
Then we have
\begin{align*}
\Delta_{t+1} &=Y_{t+1}-Y_{t+1}'\\
&=Q_{t+1}(s,k)-\delta(s,k)-Q_{t+1}(s,k')+\delta(s,k')\\
&=(1-\alpha_t)Q_t(s,k)+\alpha_t(\delta(s,k)+c_{s'})-\delta(s,k)\\
&\quad-(1-\alpha_t')Q_t(s,k')-\alpha_t'(\delta(s,k')+c_{s'})+\delta(s,k')\\
&=Q_t(s,k)-Q_t(s,k')-\delta(s,k)+\delta(s,k')\\
&\quad+\alpha_t(\delta(s,k)+c_{s'}-Q_t(s,k))\\
&\quad+\alpha_t'(Q_t(s,k')-\delta(s,k')-c_{s'})\\
&=\Delta_t+(\alpha_t- \alpha'_t)c_{s'} -\alpha_t Y_t+\alpha_t'Y_t'
\end{align*}
We use this quantity to define
$$F_t = F_t(s,k,k') = \beta_t^{-1} ((\alpha_t- \alpha'_t)c_{s'} -\alpha_t Y_t+\alpha_t'Y_t') + (Y_t - Y'_t)$$
if $S_t=s$ and $0$ otherwise.

We take $P_t$ to be the $\sigma$-field generated by $H_t$.
Observe that $Q_t$ and $\beta_t$ are measurable given $H_t$.
Thus $\Delta_t$ and $\beta_t$ are measurable given $P_t$ and $F_t$ is measurable given $P_{t+1}$.

If we can show $\Delta_t$ converges a.\,s.\ to $0$, we are done, so it suffices to prove the conditions of Theorem~\ref{thm:sa-convergence}.

\textbf{Condition 1:}
When $S_t=s$, the statement holds since $F_t=\beta_t^{-1}(\Delta_{t+1}-\Delta_t) + \Delta_t$. If $s\neq{S_t}$, we have $\alpha_t=0$ and so $\beta_t=0$, hence $\Delta_{t+1}=(1-\beta_t)\Delta_t+\beta_t F_t$ as required.

\textbf{Condition 2:}
By Lemma \ref{lem:expectation-learning-rate-epsilon}, we have $\beta_t= \ev[\alpha_t \mid H_t] = \lrscale M_t(s)^{-1}$ when $S_t = s$ and $0$ otherwise. Then we have
$$\sum_{t=0}^\infty \beta_t = \sum_{t=0}^\infty \mathbb{I}_{S_t=s}\frac{\lrscale}{M_t(s)} = \lrscale \sum_{k=0}^\infty \frac{1}{k} = \infty$$
$$\sum_{t=0}^\infty \beta_t^2 = \lrscale^2 \sum_{k=0}^\infty \frac{1}{k^2} < \infty.$$

\textbf{Condition 3:}
Let $t\geq 0$ be an arbitrary integer. If $S_t=s$, we have
\begin{align*}
\ev[F_t\mid P_t]
&= \ev[ \beta_t^{-1} ( (\alpha_t - \alpha'_t)c_{s'} -\alpha_t Y_t + \alpha_t'Y_t' ) + (Y_t - Y'_t)\mid H_t]\\
&= \beta_t^{-1} \left( \ev[(\alpha_t - \alpha'_t) \mid H_t] \cdot \ev [c_{s'} \mid H_t] - \ev[\alpha_t \mid H_t] Y_t +  \ev[\alpha_t' \mid H_t] Y'_t\right) + (Y_t - Y'_t) \\
&= \beta_t^{-1} ( 0 \cdot \ev [c_{s'} \mid H_t] - \beta_t Y_t + \beta_t Y'_t) + (Y_t - Y'_t) \\
&= (-Y_t + Y'_t) + (Y_t - Y'_t) \\
&= 0
\end{align*}
On the second line, the assumption of \uniformCorruptions{} is what allows us to factor out the expectation of $c_{s'}$.
On the third line, we use the fact that $\ev[\alpha_t\mid H_t]=\ev[\alpha'_t\mid H_t]=\beta_t$ as in the previous condition.

If $S_t\not=s$, $F_t=0$, so $\ev[F_t\mid P_t] = 0$ as well. Thus,
$$\norm{\ev[F_t \mid P_t]}_\infty = 0 \leq \gamma \norm{\Delta_t}_\infty.$$

\textbf{Condition 4:}
We are required to show that there is some constant $C$ such that $\Var[F_t\mid P_t]\leq C(1+\norm{\Delta_t}_\infty)^2$, so it is sufficient to find $C>0$ such that $\Var[F_t\mid P_t]\leq C$, i.e., that the conditional variance of $F_t$ is bounded.
This holds by Popoviciu's inequality, as long as the random vector $F_t$ is bounded given the past.
We now show that it is bounded even without conditioning on the past.

As per Lemma~\ref{lem:alpha-bounded}, we have $\alpha_t(s,k)\in[0,1]$ for all $t, s, k$.
We also know that $\delta(s,k)$ and $c_{s'}$, and their sum, are bounded for all $t, s, k$, since $\mcS\times\mcK$ is finite.
It follows from Lemma~\ref{lem:Q-bounded}, applied element-wise, that $Q_t(s,k,k')$ is bounded for all $t,s,k,k'$.
Therefore $Y_t$ and $Y_t'$ are also bounded for all $t,s,k,k'$.

Finally, we have
$$\frac{\alpha_t}{\beta_t} = \frac{1}{\pi_K(k|s)} < \frac{1}{\lrscale}$$
so $\frac{\alpha_t}{\beta_t},\frac{\alpha_t'}{\beta_t}$ are bounded.
Hence, $F_t=\frac{\alpha_t'}{\beta_t}Y_t-\frac{\alpha_t}{\beta_t}Y_t'+\frac{\alpha_t}{\beta_t}c_{s'}-\frac{\alpha_t'}{\beta_t}c_{s'}$ is bounded at all times and indices.
\end{proof}

\begin{lemma}\label{lem:expectation-learning-rate-epsilon}
For all state-action pairs $(s,k)$, the expected value of $\alpha_t$, conditioned on the past, is the reciprocal of the state visit count for the current state, and zero otherwise:
\begin{align*}
\beta_t = \ev\left[\alpha_t(s,k)\mid H_t\right]&=\begin{cases}
\lrscale/M_t(s)&S_t=s\\
0&S_t\neq s
\end{cases}
\end{align*}
\end{lemma}
\begin{proof}
First suppose $S_t=s$.
Note that we have $\mathbb{P}[(S_t,K_t)=(s,k)\mid H_t]=\pi_K(K_t=k\mid H_t)$, since, given $H_t$, we know $S_t=s$ and $K_t$ depends only on the query policy.
By the definition of expected value, summing over the two possible values for $\alpha_t(s,k)$ in its definition, we have
\begin{align*}
    \ev\left[\alpha_t(s,k)\mid H_t\right]
    &=\lrscale\left(M_t(s)\,\pi_K(K_t=k\mid H_t\right)^{-1}\mathbb{P}\left[(S_t,K_t)=(s,k)\mid H_t\right]\\
    &\quad+0\cdot\mathbb{P}[(S_t,K_t)\neq(s,k)\mid H_t]\\
    &=
    \lrscale\left(M_t(s)\,\pi_K(k|s)\right)^{-1}\times\pi_{K}(k|s)\\
    &=\lrscale\left(M_t(s)\right)^{-1}
\end{align*}
as required.
In the case $S_t \neq s$, we have $\mathbb{P}\left[(S_t,K_t)=(s,k)\mid H_t\right]=0$, so
$\ev\left[\alpha_t(s,k)\mid H_t\right]=0$.
\end{proof}

\begin{lemma}\label{lem:alpha-bounded}
The learning rates are bounded: $\alpha_t(s,k)\in[0,1]$ for all $t,s,k$.
\end{lemma}
\begin{proof}
By definition, $\alpha_t(s,k)$ is either $0$ or $\lrscale(M_t(s)\pi_K(k|s))^{-1}$.
Since $\pi_K$ is $\epsilon$-greedy with $\epsilon = \max(1/M_t(s),  \lrscale \size{A})$,
we have $1\geq\pi_K(k|s)>\lrscale>0$.
Thus $\lrscale (\pi_K(k|s))^{-1}\in(0,1)$.
Since $M_t(s)$ a positive integer, multiplication by $M_t(s)^{-1}\in(0,1]$ simply scales within the unit interval.
\end{proof}

\begin{lemma}\label{lem:Q-bounded}
Let $\alpha_t$, $\tilde D_t$, and $Q_t$ be sequences of random variables ($t\geq 0$ an integer) such that $\alpha_t$ is bounded by $[0,1]$ and $\tilde D_t$ is bounded, i.e., there are $\overline{D}$ and $\underline{D}$ such that $0\leq\alpha_t\leq 1$ and $\underline{D}\leq \tilde D_t\leq\overline{D}$ for all $t$.
Further, suppose the following equation is satisfied for all $t$:
\begin{align*}
Q_{t+1}=(1-\alpha_t)Q_t+\alpha_t \tilde D_{t+1}
\end{align*}
Also, suppose $Q_0$ is bounded, and the bounds, without loss of generality, are $[\underline{D},\overline{D}]$.
Then the variables $Q_t$ are bounded for all $t$ by $[\underline{D},\overline{D}]$.
\end{lemma}
\begin{proof}
By induction on $t$, we have the base case by assumption.
Now assume $\underline{D}\leq Q_t\leq\overline{D}$ as inductive hypothesis.
First let's prove the upper bound by cases on whether $Q_t$ or $\tilde D_{t+1}$ is smaller.
If $Q_t\leq \tilde D_{t+1}$ then
\begin{align*}
(1-\alpha_t)Q_t+\alpha_t \tilde D_{t+1}&\leq (1-\alpha_t)\tilde D_{t+1}+\alpha_t \tilde D_{t+1}\\
&=\tilde D_{t+1}\\
&\leq\overline{D}
\end{align*}
where the first line holds because $(1-\alpha_t)\geq 0$.
If $\tilde D_{t+1}< Q_t$, then $(1-\alpha_t)Q_t+\alpha_t \tilde D_{t+1} < (1-\alpha_t)Q_t + \alpha_t Q_t=Q_t\leq \overline{D}$.
The lower bound is similar.
\end{proof}

\section{Analysis of importance sampling correction in simple CFMDP}\label{app:importance-sampling}

This section explains why the importance sampling correction in DA-QL is necessary, and along the way illustrates a subtle source of tampering incentives: \emph{confounding}.
Our example shows that even in a simple setting with \uniformCorruptions, decoupling is necessary to avoid tampering incentives.
This motivates rigorous analysis of incentives: even in cases where the action being optimized cannot causally affect the feedback, tampering incentives can arise through indirect pathways -- in this case, confounding.

We provide a simple CFMDP satisfying \uniformCorruptions, Example \ref{ex:simple_cfmdp}, in which DA-QL without the importance sampling correction fails to converge to the optimal policy.
In this CFMDP, the first action receives higher true approval but the second action receives higher observed approval.

\begin{example}
\label{ex:simple_cfmdp}
  Consider the CFMDP with $\mcS = \mcA = \{x^0,x^1\}$, a deterministic transition function such that $S_{t+1} = A_t$, and additive corruption $c_{x^s} = 10s$.
The agent is trained from approval with $\mcK = \mcA$ and feedback given by $\delta(\cdot, x^0) = 1$ and $\delta(\cdot, x^1) = 0$.
\end{example}

Figure \ref{fig:adql_example} shows the results of applying DA-QL without the importance sampling correction, by changing the learning rate in Algorithm \ref{alg:adql} from $\lrscale(M_t(s) \pi_K(s, k))^{-1}$ to $\lrscale M_t(s)^{-1}$.

\begin{figure}[h]
    \centering
    \includegraphics[width=0.7\columnwidth]{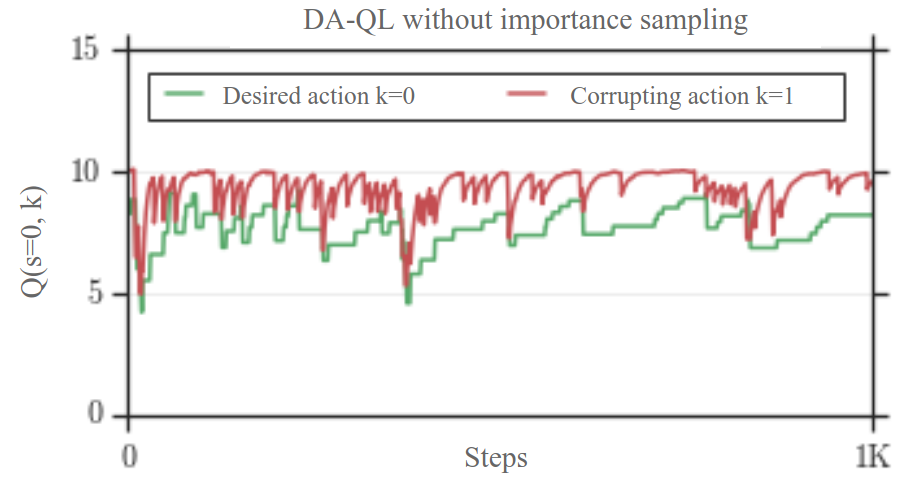}
\vspace{-4mm}
\caption{Learning dynamics of DA-QL without importance sampling.
Empirically, Q-values for the corrupting action (red) remain above Q-values for the desired action (green).
The Q-values of the desired action steadily increase, until exceeding the Q-values for the corrupting action, which is followed by a sudden drop.
The main text explains this observation in terms of confounding.
}
\label{fig:adql_example}
\end{figure}

In this example, at each step, the expected observed approval always favors the desired action $x^0$ over the corrupting action $x^1$, conditional on the current policy.
This follows from \uniformCorruptions: the amount of expected corruption does not depend on the query, so the expected observed approval for each query is just the true approval shifted by the same expected corruption.

However, perhaps surprisingly, we observe empirically that the learned Q-values favor the corrupting action, $x^1$, on almost all steps (Figure \ref{fig:adql_example}).
We explain this observation via \emph{confounding}: the policy is a common cause for both the query and taken actions, and thus acts as a confounding variable.
In other words, while the query and taken actions are independent conditional on the policy, they become dependent after marginalizing out for variations in the policy throughout training.
The result is that the corrupting action $x^1$ is more likely to be queried on steps where the corrupting action is also taken.
This causes queries about the corrupting action to become correlated with positive updates to the Q-values, despite having no direct causal influence on the corruption.

In Figure \ref{fig:adql_example}, the negative spikes in $Q(s=x^0, k=x^0)$ arise due to the flip-side of this effect -- whenever the relative ordering of Q-values switches to preferring the desired action $x^0$, that action becomes more likely to be both the queried and taken.
The absence of corruption on these steps then causes the observed negative spike.

Importance sampling counterbalances this effect by applying larger updates inversely proportional to the probability of updating each particular query.
This ensures that the expected Q-update at each step favors the desired action, as shown in Proposition \ref{prop:adql_local_incentives}.

To explain this effect in greater detail, we now provide additional plots over various timescales, as well as analogous results for DA-QL with importance sampling and for DA-PG, which both converge to the optimal policy as expected.

\subsection{DA-QL}

\begin{figure}[h]
\centering
\begin{subfigure}{0.48\textwidth}
    \includegraphics[width=\textwidth]{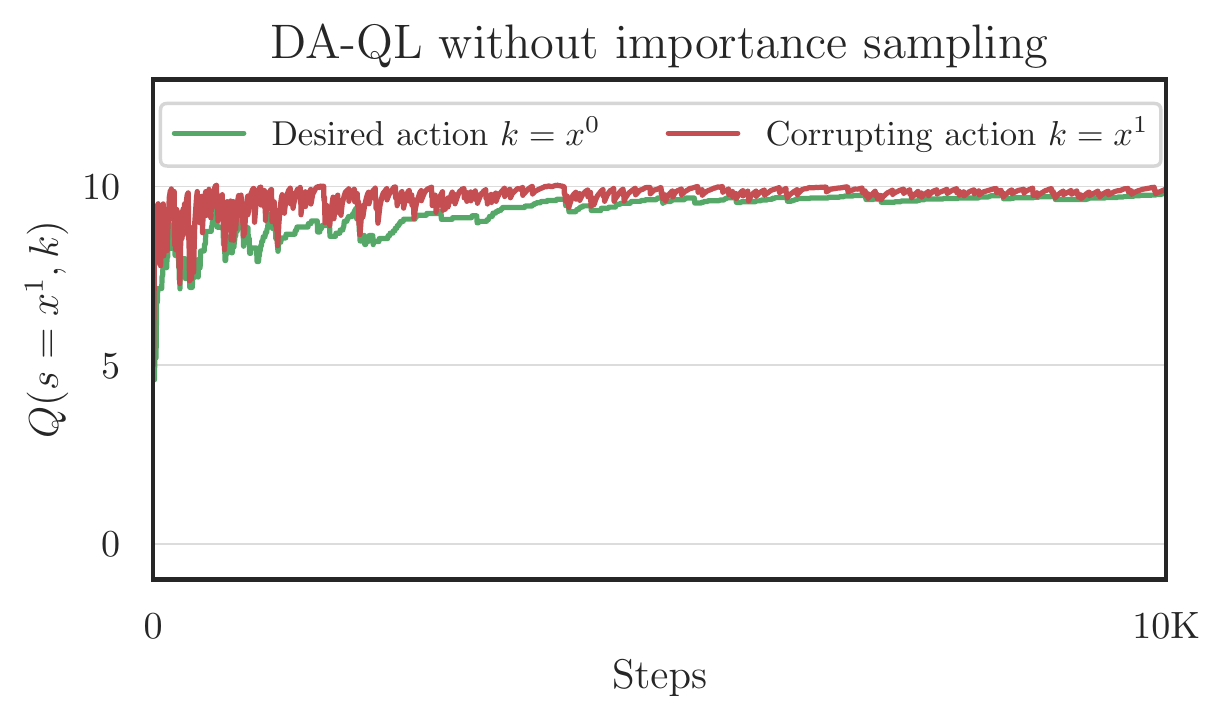}
\end{subfigure}
\begin{subfigure}{0.48\textwidth}
    \includegraphics[width=\textwidth]{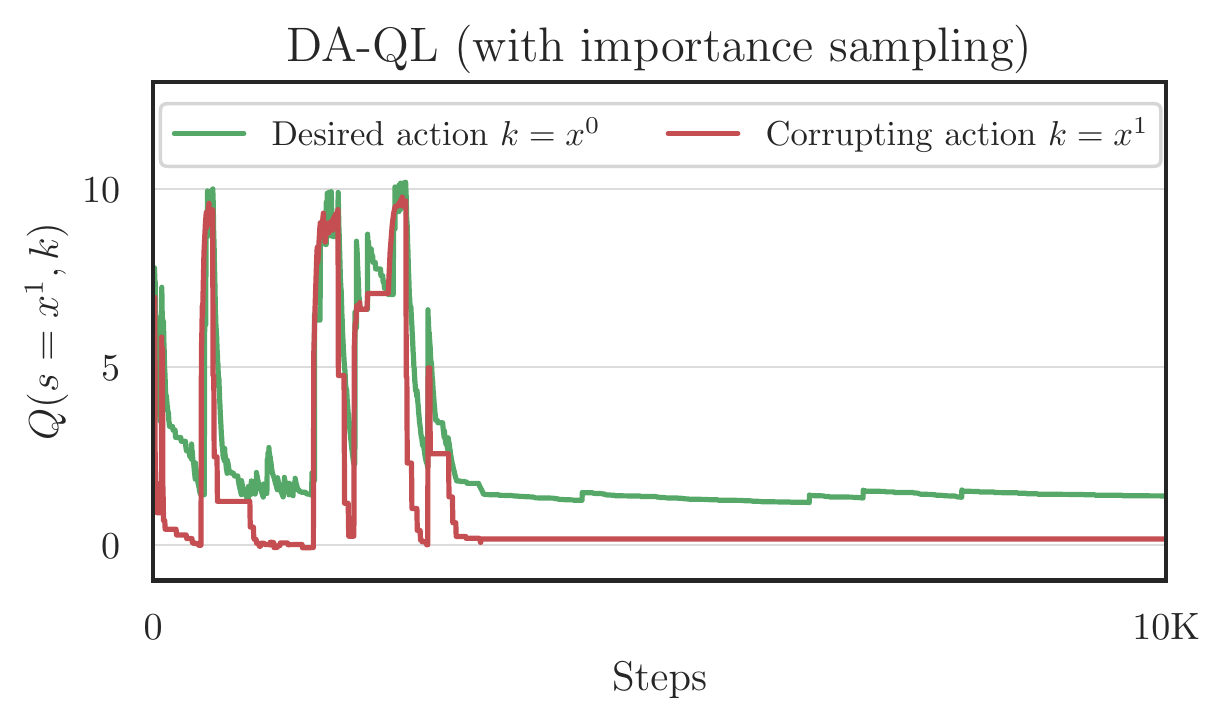}
\end{subfigure}
\caption{DA-QL without and with importance sampling.
As discussed in the main text, without the importance sampling correction, Q-values for the corrupting action (red) remain above Q-values for the desired action (green).
With the importance sampling correction, this pattern is reversed: the Q-values favor the desired action, and converge to the true Q-values of $1$ and $0$ respectively.
}
\label{fig:adql_example_extra}
\end{figure}

\textbf{Importance sampling correction produces aligned incentives.}
Figure \ref{fig:adql_example} illustrates that, with the importance sampling correction ablated, DA-QL may fail to learn the desired policy.
In Figure \ref{fig:adql_example_extra}, we provide side-by-side plots of the Q-values over time for DA-QL with and without the importance sampling correction.
We observe that with the importance sampling correction, DA-QL reliably converges to the optimal policy and true Q-values.

Note that this observation is more or less immediate from aligned local update incentives (Proposition \ref{prop:adql_local_incentives}) -- with the importance sampling correction, the expected Q-update at each step is more positive for the desired action $x^0$, whereas without the importance sampling correction, this can be reversed.
It is thus expected that with the importance sampling correction, DA-QL learns to favor the desired action $x^0$.

\begin{figure}[h]
\centering
\begin{subfigure}{0.48\textwidth}
    \includegraphics[width=\textwidth]{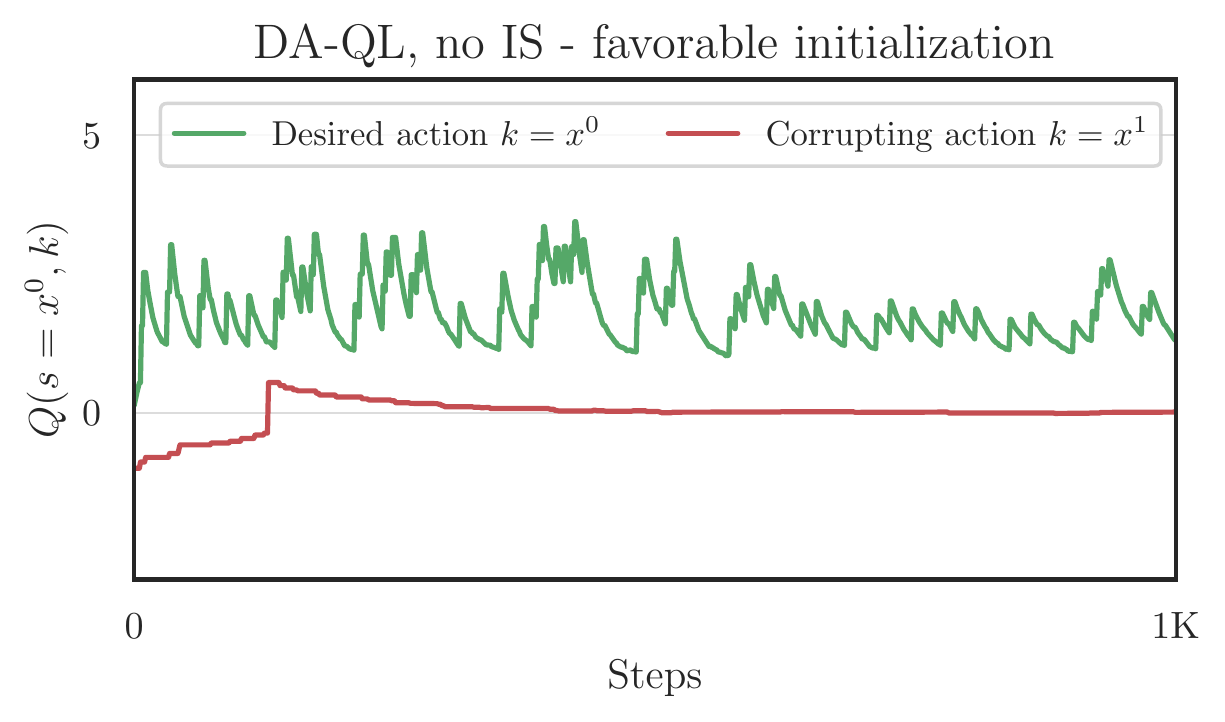}
\end{subfigure}
\begin{subfigure}{0.48\textwidth}
    \includegraphics[width=\textwidth]{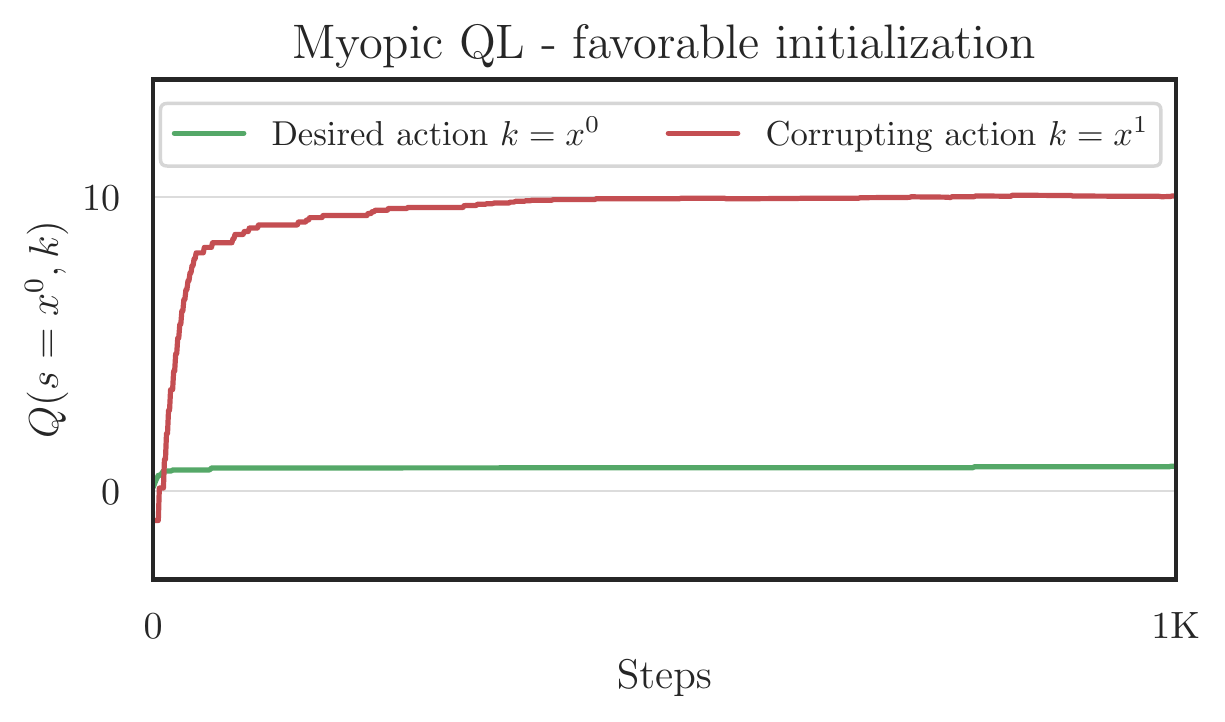}
\end{subfigure}
\caption{DA-QL without importance sampling correction, but favorable initialization.
With favorable initializations, DA-QL without importance sample correction can still learn the desired behavior, in contrast to \Approvalrl{} QL which immediately converges to a tampering policy.
For these plots, we initialize the Q-values of the desired and corrupting action to $0$ and $-1$ respectively, rather than initializing both to $5$ as in other figures.
}
\label{fig:adql_good_init}
\end{figure}

\textbf{Even without the importance sampling correction, decoupled feedback can still produce desired behavior.}
Figure \ref{fig:adql_good_init} shows that decoupled feedback confers significant benefits on its own.
In particular, if we initialize Q-values such that tampering is unlikely under the initial policy, DA-QL even without importance sampling will frequently converge to the desired policy.
This is most easily understood by contrast with \Approvalrl{} QL - as seen through Remark \ref{remark:myopic}, each time \Approvalrl{} QL takes the corrupting action, this directly results in high observed reward.
In contrast, when using decoupled feedback, the agent will only observe high reward if it also takes the corrupting action.
This may happen quite rarely if tampering is unlikely for the initial policy.
More generally, the correlation introduced by confounding might be sufficiently weak that it is dominated by the advantage in true approval feedback.

While contingent on details such as initialization and task-specific features, we note that in many real-world tasks, this property may provide a significant advantage for decoupled feedback even without full decoupling.
Especially if the learning environment is designed with checks against tampering, tampering may be unlikely under the initial policy.

\textbf{Quantitative statistics.}
To provide an overall sense of learning behavior, we evaluate 100 runs each of DA-QL, both with and without the importance sampling correction, and \Approvalrl{} QL, where the initial Q-values are sampled from a Gaussian distribution with mean $0$ and standard deviation $3$.
For each run, we measure the fraction of steps on which the Q-values favor the desired action.
This distribution is bimodal, since in each run, the algorithm converges to either the desired or corrupting action.
For DA-QL, in 98\% of runs, the desired action is favored over 80\% of the time (and in all runs, over 50\%).
For DA-QL without the importance sampling correction, this happens in 15\% of runs, and for \Approvalrl{} QL, 0\% of runs.
In short, DA-QL produces the desired policy, \Approvalrl{} QL produces a tampering policy, and for DA-QL without the importance sampling correction, it depends on initialization.

\subsection{DA-PG}

\begin{figure}[h]
\centering
\begin{subfigure}{0.48\textwidth}
    \includegraphics[width=\textwidth]{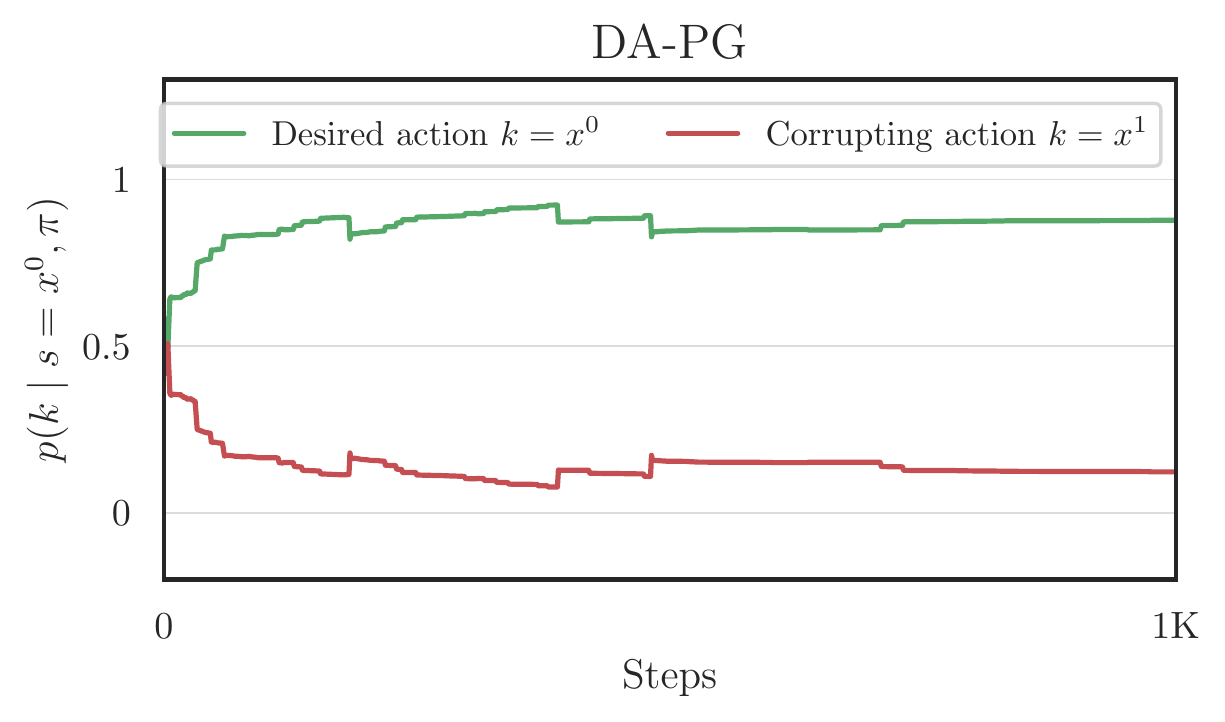}
\end{subfigure}
\begin{subfigure}{0.48\textwidth}
    \includegraphics[width=\textwidth]{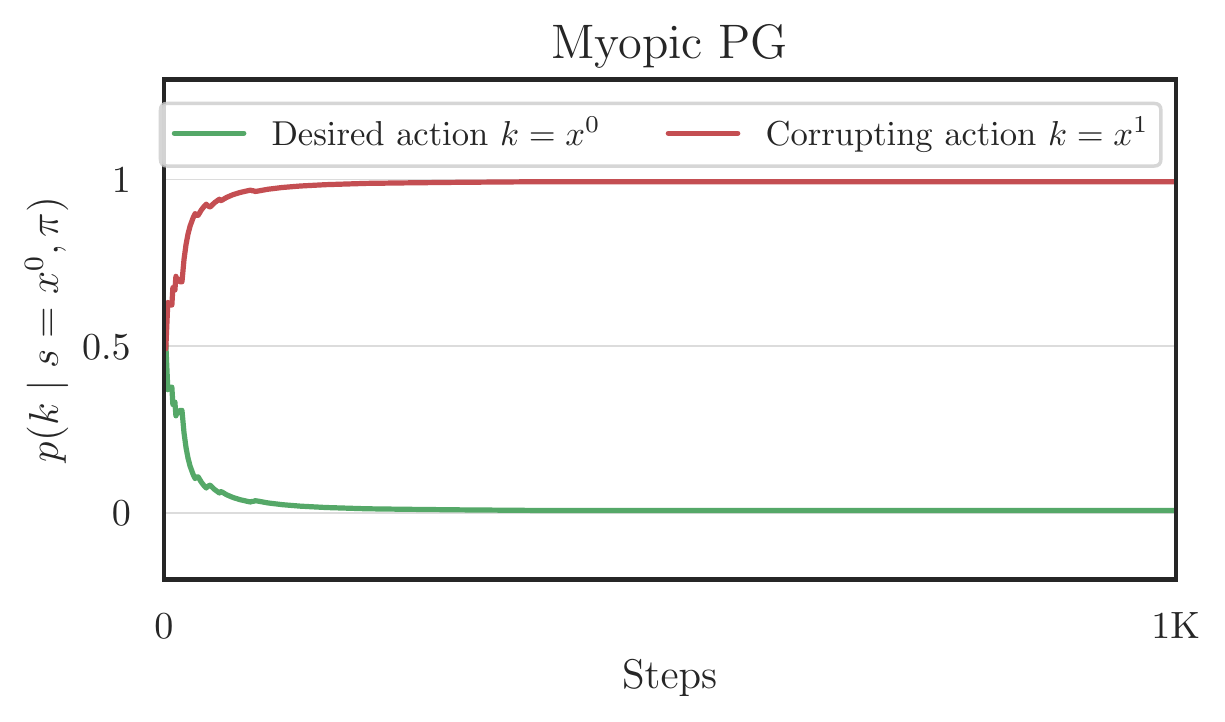}
\end{subfigure}
\caption{Behavior of DA-PG. DA-PG quickly learns to prefer the desired action, whereas \Approvalrl{} PG quickly learns to prefer the corrupting action.}
\label{fig:adpg_simple_cfmdp}
\end{figure}

For completeness, we also report results for DA-PG in this simple CFMDP.
We represent the policy with tabular Q-values, which define a Boltzmann or softmax policy, and add an entropy bonus.
Figure \ref{fig:adpg_simple_cfmdp} shows that DA-PG quickly learns to prefer the desired action, whereas \Approvalrl{} PG learns to prefer the corrupting action.
Quantitatively, we again perform 100 runs of each agent using the same initialization distribution used for DA-QL.
We observe that in 95\% of runs, DA-PG favors the desired action over 80\% of the time (for very unfavorable initializations, there is insufficient exploration for the agent to ever learn).
This happens 0\% of the time for \Approvalrl{} PG.

\section{Experiments on procedural tabular CFMDPs}
\label{app:tabular_cfmdp}

We include additional experiments on procedurally generated tabular CFMDPs.
We run on 1000 such CFMDPs, each with $|\mcS| = 10, |\mcA| = 4$.
For each CFMDP, the underlying MDP uses a deterministic reward function, fixed by sampling from a log-normal distribution independently for each state and action.
The transition matrix is sampled from a Dirichlet distribution, i.e., for each state-action pair, the distribution over resulting states is sampled from a Dirichlet.
We then train a Q-learning agent in the underlying (uncorrupted) CFMDP which we use as the approver for our approval-based experiments.

The corruption function is given by the probabilities of a Dirichlet distribution, multiplied by a large scale to ensure that the scale of corruptions is larger than the scale of rewards.
We then train both a DA-PG and \Approvalrl{} PG agent using the approval signal.
We consider the agent successful if the agent's total return according to the true reward function exceeds 80\% of the return achieved by the approver.
Across the 1000 CFMDPs, DA-PG succeeds in 100\% of CFMDPs, whereas \Approvalrl{} PG succeeds in only 35\%.
As expected, \Approvalrl{} PG achieves higher corrupted approval than DA-PG on average due to receiving positive corruptions, while achieving lower true approval.

\section{Additional Experiments}
\label{app:additional_exps}

\subsection{Results for Seek-Avoid}
\label{app:seekavoid}

\begin{table}[ht]
\centering
  \begin{tabular}{p{0.3\columnwidth}cc}
\toprule
    \textbf{Algorithm} &  \thead{\textbf{Tampering} \\ \textbf{(\% Episodes)}} & \thead{\textbf{Return} \\ \textbf{(Mean)}}\tabularnewline
\midrule
    Standard RL & 68\% & \hp1.6 \tabularnewline
    Non-embedded & \hp1\% & 14.1 \tabularnewline
\midrule
    \Approvalrl{} QL & \hp1\% & 10.0 \tabularnewline
    DA-QL (no IS corr.) & \hp1\% & 11.6 \tabularnewline
    DA-QL & \hp3\% & 13.4 \tabularnewline
\midrule
  \Approvalrl{} PG & \hp1\% & 14.4 \tabularnewline
    DA-PG & \hp1\% & 13.7 \tabularnewline
\bottomrule \tabularnewline
\end{tabular}
\caption{\textbf{Seek-Avoid results.}
All agents other than the Standard RL agent perform fairly well on true return for Seek-Avoid.
We find that all myopically optimized agents do not tamper, which we attribute to using a tampering-avoiding approver, whereas the Unlock Door experiments use a tampering-agnostic approver.
}
\label{tab:seekavoid_results}
\end{table}

We evaluate our main approaches from \S\ref{subsec:comparisons} on Seek-Avoid, a simple task involve gathering positive reward objects and avoiding negative reward objects.
Unlike our main experiments, which use a tampering-agnostic approver, for these experiments, we train the approver Q-network to avoid tampering, by providing a tampering penalty based on oracle access to the true and observed reward.
Our results for Seek-Avoid are summarized in Table \ref{tab:seekavoid_results}.
For Seek-Avoid, we find that approval-based agents are also able to solve the task and obtain high true return.
We further find that \approvalrl{} RL agents avoid tampering.
We observe that on Seek-Avoid, tampering-agnostic approvers tamper often, as do agents trained using this approval signal, so we attribute the difference to the use of a tampering-avoiding approver.
As discussed in \S\ref{sec:myopic_sufficiency}, incentives in myopically optimized agents are dependent on the optimal trajectories.
For Seek-Avoid specifically, it is possible that the tampering-avoiding approval signal is sufficient to avoid tampering incentives for myopically optimized \approvalrl{} RL agents, or just that it increases the difficulty of finding tampering solutions.

\subsection{Imitation Learning Results}
\label{app:imitation_exps}
For imitation learning, we ran experiments with both embedded and non-embedded feedback.
While the agents performed well with non-embedded feedback, with embedded feedback, we observed significantly worse performance (roughly half the return for DAgger, and a quarter for Offline Imitation).
We believe this is largely attributable to the specifics of our Register implementations.
Random actions can frequently cause a constant corruption to the feedback throughout an episode.
For approval-based algorithms, this adds noise but does not affect the advantage estimate.
However, for imitation learning with our encoding method, this causes the agent to receive incorrect demonstrator actions.
We plan to address this and investigate imitation learning with embedded feedback more carefully in future work.

\end{document}